\documentclass[10pt]{article} 

\usepackage[preprint]{rlj} 

\usepackage{amssymb}            
\usepackage{mathtools}          
\usepackage{mathrsfs}           
\usepackage{graphicx}           
\usepackage[space]{grffile}     
\usepackage{url}                
\usepackage{lipsum}             
\usepackage{xcolor}         
\usepackage[space]{grffile}     
\definecolor{ggray}{HTML}{666666}

\definecolor{rred}{HTML}{CC0000}

\definecolor{ggreen}{HTML}{000000}
\definecolor{bblue}{HTML}{000000}
\definecolor{pmdcolor}{HTML}{4151fa}
\definecolor{pmdmomcolor}{HTML}{df473d}
\definecolor{picolor}{HTML}{34ae3b}
\definecolor{orange}{HTML}{f56a29}

\usepackage{cancel}

\let\classAND\AND
\let\AND\relax
\usepackage{algorithmic}

\let\AND\classAND
\AtBeginEnvironment{algorithmic}{\let\AND\algoAND}
\usepackage{algorithm}
\floatname{algorithm}{Algorithm}

\usepackage{amssymb}
\usepackage{amsmath}
\usepackage{mathtools}
\usepackage{amsthm}
\usepackage{microtype}
\usepackage{graphicx}
\usepackage{subfigure}
\usepackage{booktabs} 
\usepackage{hyperref}
\usepackage{pifont}
\usepackage{wrapfig}
\usepackage{enumitem}
\usepackage[mathscr]{eucal}
\usepackage{enumitem}
\usepackage{graphicx}
\usepackage{subfigure}
\usepackage{booktabs} 
\usepackage{mathtools}
\usepackage{caption}
\usepackage{subcaption}
\usepackage{graphicx}
\usepackage{float}
\usepackage{multirow}
\usepackage{multicol}

\newcommand{\eq}[1]{\begin{align*}#1\end{align*}}
\newcommand{\eqq}[1]{\begin{align}#1\end{align}}

\def\E{\mathbb{E}}
\def\R{\mathbb{R}}

\usepackage{nicefrac}
\def\eps{{\epsilon}}
\def\A{\mathscr{A}}
\def\I{\mathbf{I}}
\def\R{\mathbb{R}}
\def\X{\mathcal{X}}
\def\C{\mathcal{C}}
\def\w{\mathrm{w}}

\def\argmin{\operatorname{argmin}}

\def\T{\mathcal{T}}

\def\S{\mathscr{S}}

\def \eps{{\epsilon}}

\def\Omega{{h}}
\def\PMD{{PMD}}
\def\PMDext{\texttt{PMD(+extragradient)}}

\def\PMDmom{\texttt{Lazy PMD(+momentum)}}

\def\PI{{PI}}

\def\pmdext{{PMD(+ext)}}

\def\pmdmom{{PMD(+mom)}}
\def\pmdloo{{PMD(+loo)}}

\theoremstyle{plain}
\newtheorem{theorem}{Theorem}[section]

\newtheorem{proposition}[theorem]{Proposition}
\newtheorem{lemma}[theorem]{Lemma}
\newtheorem{corollary}[theorem]{Corollary}
\theoremstyle{definition}
\newtheorem{definition}[theorem]{Definition}
\theoremstyle{definition}

\theoremstyle{remark}

\usepackage{fontawesome}

\usepackage[font=small,labelfont=bf]{caption}

\usepackage{booktabs,array}
\usepackage{tabularx}

\title{Functional Acceleration for Policy Mirror Descent}

\setrunningtitle{Functional Acceleration for Policy Mirror Descent}

\author{Veronica Chelu\textsuperscript{\ding{96} \ding{95}},
Doina Precup\textsuperscript{\ding{96} \ding{95} \ding{97} \ding{99}}
}

\emails{}

\affiliations{
$^{\text{\ding{96}}}${McGill University}, 
$^{\text{\ding{95}}}${Mila Quebec AI Institute}, \\
$^{\text{\ding{97}}}${Google DeepMind}, 
$^{\text{\ding{99}}}${CIFAR AI Chair}
}
\contribution{
    We introduce a momentum-based PMD algorithm that applies functional acceleration, leading to improved policy optimization dynamics.
    }
    {
    Extending PMD \citep{johnson2023optimal}. 
    }

\contribution{
   We theoretically characterize the properties of  functional acceleration and its impact on learning efficiency.
    }
    {
    Extending iteration convergence analysis of PMD. We analyze three acceleration updates: (i) Idealized: with lookahead. We provide iteration convergence. Our method is similar in spirit to the concurrent work of \citet{protopapas2024policymirrordescentlookahead}. (ii) Approximate (impractical): with extrapolation. We provide iteration convergence. Our method  is inspired by the classic work of \citet{Nemirovski04}. (iii) Approximate (practical): with momentum. We provide a derivation from the extrapolation method (ii), but we do not analyze the iteration convergence. 
    }

\contribution{
    We conduct numerical studies illustrating when and why functional acceleration benefits learning, particularly in ill-conditioned problem settings.
    }
    {
    We are not aware of any prior study. We reuse some experimental setups from other works: we test on randomly generated MDPs (setup similar to \citet{Archibald1995}), and show policy dynamics on the value polytope (inspired by \citet{dadashi2019value}).
    }
    \contribution{
    We analyze the effect of an inexact critic on functional acceleration and its implications for policy optimization.
    }
    {
    Extending prior theoretical analysis on Inexact PMD \citep{johnson2023optimal}. We are inspired also by previous results shown in \citep{russo2022approximation}, on the impact of value error on methods that use hard greedification.
    }


\summary{We investigate functional acceleration in the context of Policy Mirror Descent (PMD), a general class of policy optimization methods in Reinforcement Learning (RL). Leveraging momentum in the dual policy space, we introduce a novel momentum-based surrogate objective that accelerates convergence. Taking a functional perspective, our approach is independent of the policy parameterization and distinct from prior applications of momentum to the policy parameter. We theoretically analyze its properties and perform an extensive numerical study to evaluate its practical implications. Our results indicate functional acceleration improves efficiency in certain settings characterized by ill-conditioning. We also explore the impact of using an inexact critic on learning dynamics and provide insights into the robustness of the approach to value estimation.
}

\begin{document}

\maketitle  

\begin{abstract}
We apply \emph{functional acceleration} to the Policy Mirror Descent (PMD) general family of algorithms, which cover a wide range of novel and fundamental methods in Reinforcement Learning (RL). Leveraging duality, we propose a momentum-based PMD update.
 By taking the functional route, our approach is independent of the policy parametrization and applicable to large-scale optimization, covering previous applications of momentum at the level of policy parameters as a special case. 
We theoretically analyze several properties of this approach and complement with a numerical ablation study, which serves to illustrate the policy optimization dynamics on the value polytope, relative to different algorithmic design choices in this space. We further characterize  numerically several features of the problem setting relevant for functional acceleration, and lastly, we investigate the impact of approximation on their learning mechanics\footnote{Code is available at \href{https://github.com/veronicachelu/functional-acceleration-for-pmd}{https://github.com/veronicachelu/functional-acceleration-for-pmd}}.
\end{abstract}

\section{Introduction}\label{sec:introduction}
The RL framework \citep{rl_book} refers to the problem of solving sequential decision making tasks under uncertainty, together with a class of solution methods tailored for it. The RL problem has found applications in games \citep{Tesauro1994, MnihKSGAWR13, silver2014, Mnih2016, 
 silver2017, Hessel2017, BellemareDM17, Schrittwieser2019, zahavy2023diversifying}, robotic manipulation \citep{Schulman15,Schulman17,Haarnoja18}, medicine \citep{alphafold, Schaefer2004, nie2020learning} and is formally described by means of discounted Markov Decision Processes (MDPs) \citep{puterman}. On the solution side, increased interest has been devoted to the study of policy-gradient (PG) methods based on optimizing a parameterised policy with respect to an objective \citep{Williams92, Konda1999, Sutton2000, Agarwal19, BhandariandRusso_globalguarantees, kakade2001, bhandari21a, Mei20a, mei20b}.

Policy Mirror Descent (PMD) \citep{Agarwal19, bhandari21a, xiao2022convergence, johnson2023optimal, Vaswani2021} is a general family of algorithms, specified by the choice of mirror map covering a wide range of novel and fundamental methods in RL. 
PMD is a proximal algorithm \citep{parikh2014proximal} and an instance of Mirror Descent (MD) \citep{BeckT03} on the policy simplex \citep{bhandari21a}. MD is an iterative optimization algorithm that extends gradient descent (GD) by adapting to the geometry of the problem using different distance functions, particularly Bregman divergences  \citep{Amari1998, Bubeck15, Banerjee2005}. PMD applies a proximal regularization, with a Bregman divergence, to the improvement step of Policy Iteration (PI), and converges to it as regularization decreases.  In the discounted setting, with an adaptive step-size, PMD converges linearly at a rate determined by the discount factor, independent of the dimension of the state space or problem instance \citep{johnson2023optimal}, recovering classical approaches, like PI and VI, as special cases. 
 PMD has been extended to linear approximation by \citet{yuan2023linear} and to general function approximation by \citet{alfano2024novel}.

\textbf{Motivation}\quad 
The running time of PMD algorithms scales with the number of iterations, and with a parametrized policy class, each iteration of an approximate PMD method generally requires multiple ``inner-loop'' updates to the policy parameter \citep{Tomar20, Vaswani2021}. Actor-critic (AC) methods \citep{Sutton2000, Konda1999} additionally require the computation or updating of an inexact critic at each iteration. 
It is therefore desirable to design algorithms which converge in a smaller number of iterations, resulting in significant empirical speedups, as has been previously argued by \citet{johnson2023optimal, xiao2022convergence, goyal2021firstorder, russo2022approximation}. 

\textbf{In this work,} we leverage duality and acceleration to build a novel surrogate objective for momentum-based PMD, leading to faster learning in terms of less iterations necessary to converge.
The novelty of our approach is the application of acceleration mechanics to the \emph{direct} or \emph{functional policy representation $\pi$}---hence named \textbf{\emph{functional acceleration}}, as opposed to classic acceleration applied to the policy parameter $\theta$ (e.g.,  \citet{Mnih2016, Hessel2017, Schulman17} use Adam \citep{adam} or RMSProp \citep{Hinton2012}). 
Specifically, we use momentum in the dual policy space to accelerate on ``long ravines'' or decelerate at ``sharp curvatures'' at the functional level of the policy optimization objective. Intuitively, adding momentum to the functional PG (the gradient of the policy performance objective with respect to the direct policy representation $\pi$) means applying, to the current directional policy derivative, a weighted version of the previous policy ascent direction, encouraging the method to adaptively accelerate according to the geometry of the optimization problem. 

\textbf{Contributions}\quad
\begin{itemize}
    \item We illustrate and analyze theoretically the impact of applying functional acceleration on the optimization dynamics of PMD, leading to a practical momentum-based PMD algorithm.
    \item We 
    characterize the properties of the problem setting, and those intrinsic to the algorithm, for which applying functional acceleration is conducive to faster learning.
    \item We study the influence of an inexact critic on the acceleration mechanism proposed.
\end{itemize}

\textbf{Outline}\quad
This document is organized as follows. After placing our work in existing literature in Sec.~\ref{Related Work}, and setting up the context in which it operates in Sec.~\ref{sec:background_and_notation}, we introduce our main ideas in  Sec.~\ref{sec:Functional Acceleration for PMD}. We complement with numerical studies in Sec.~\ref{sec:Numerical studies}, ending with a short closing in Sec.\ref{Closing}.

\section{Related Work} \label{Related Work}

\textbf{Accelerated optimization methods}
have been at the heart of convex optimization research, e.g.,
 Nesterov's accelerated gradients (NAG) \citep{Nesterov1983AMF, WangAbernethy2018,Wang2021}, extra-gradient (EG) methods \citep{Korpelevich1976TheEM},  mirror-prox \citep{Nemirovski04, juditsky2011solving}, optimistic MD \citep{RakhlinS13, joulani20a}, AO-FTRL \citep{rakhlin2014online, mohri2015accelerating}, Forward-Backward-Forward (FBF) method \citep{Tseng1991}.

 As far as we know, our idea of applying acceleration to the direct (functional) policy representation $\pi_{\theta}$---independent of the policy parametrization $\theta$---is novel. This is important because it means universality of the approach to any kind of parametrization and functional form a practitioner requires. Within the context of RL, acceleration has only been applied to value learning \citep{Vieillard2019, farahmand21a, goyal2021firstorder}, or in the context of PG methods, classic acceleration is applied to the policy parameter $\theta$---all recent deep RL works (e.g. \citet{Mnih2016, Hessel2017, Schulman17})  use some form of adaptive gradient method, like Adam \citep{adam} or RMSProp \citep{Hinton2012}. The idea of acceleration generally relies on convexity of the objective relative to the representation of interest. The transformation from parameters $\theta$ to functional representation of the policy as probabilities $\pi_{\theta}$, can be highly complex, non-linear, and problem-dependent. Proximal algorithms operate on this functional representation, and rely on relative-convexity and relative-smoothness \citep{lu2017relativelysmooth} (relative to the mirror map $h$) of the objective with respect to $\pi$ when constructing surrogate models \citep{BhandariandRusso_globalguarantees, bhandari21a, Agarwal19, Vaswani2021}. These properties suggests the functional acceleration mechanism is feasible and promising in our setting, since it is able to successfully accelerate convex optimization \citep{joulani20a}.

 \textbf{Approximate PMD} \quad PMD has been extended to function approximation by \citet{Tomar20}, and later analyzed by \citet{Vaswani2021}, who treat the surrogate objective as a nonlinear optimization problem, that of approximately minimizing at each iteration, a composite proximal objective, denoted $\ell({\pi_{\theta}})$ with respect to the policy parameter $\theta$. In contrast, \citet{alfano2024novel, xiong2024dualapproximationpolicyoptimization} apply the PMD update in the dual form, as (generalized) Projected GD (PGD), i.e. a gradient update in the dual space followed by a projection \citep{Bubeck15},
 rather than in proximal form \citep{BeckT03}, as used by \citet{Tomar20}. Here, we use on the proximal objective.
 
\textbf{Limitations, Related \& Future Work}\quad Our focus is on developing a foundation that motivates further study. 
A translation to practical large-scale implementations and deep RL remains for further investigation, i.e. with non-standard proximal methods, e.g., TRPO \citep{Schulman15}, PPO \citep{Schulman17}, MDPO \citep{Tomar20}, MPO \citep{Abdolmaleki18}. Additional guarantees of accelerated convergence for general policy parametrization using the dual policy norm, as well as theoretical analysis for the stochastic setting, are also deferred for future work. We also note the concurrent work of \citet{protopapas2024policymirrordescentlookahead}, who also propose a version of PMD with lookahead (Sec.~\ref{sec:Functional Acceleration for PMD}), and show a similar result to Thm.~\ref{Functional acceleration with pmdloo}, beyond which the works diverge.

\section{Background \& Preliminaries}\label{sec:background_and_notation}

\textbf{RL}\quad We consider a standard RL setting described by means of a Markov decision process (MDP) $\left(\S, \A, r, P, \gamma, \rho\right)$, with state space $\S$, action space $\A$, discount factor $\gamma\in [0, 1)$,  initial state distribution $\rho\in\Delta(S)$ ($\Delta(\mathcal{X})$---the probability simplex over a set $\mathcal{X}$), rewards are sampled from a reward function $R \sim r(S,A), r:\S\times\A\to[0,R_{\max}]$, and next states from a transition probability distribution $S^\prime \sim P(\cdot|S, A) \in \Delta(\S)$.
The RL problem \citep{rl_book} consists in finding a policy $\textstyle \pi \!:\! \S\!\to\! \Delta_{\A} \!\in\! \Pi \!\equiv\! \Delta_{\A}^{|\S|}$, maximizing the performance objective defined as the discounted expected cumulative reward $V^\rho_\pi\doteq \E_{s\sim\rho}V_\pi(s) \in \R$,
where $\textstyle V_\pi \!\in\!\R^{|\S|}$ and $\textstyle Q_\pi\!\in\! \R^{|\S|\times|\A|}$ are the value and action-value functions of a policy $\pi$, such that $V_\pi(s) = \mathbb{E}_{\pi}\left[\sum_{i=0}^\infty\gamma^i R_{i+1}|S_0 = s\right]$, $Q_\pi({s,a}) \doteq \E_\pi \left[\sum_{i=0}^\infty \gamma^i R_i|S_0= s, A_0=a\right]$ and $V_\pi(s) \doteq \E_{a\sim\pi}\left[Q(s,a)\right]$.
There exists an optimal deterministic policy  $\pi^*$ that simultaneously maximizes $V_{\pi}$ and $Q_{\pi}$ \citep{bellman}. 
Let $d_{\pi}$ be the discounted visitation distribution $\textstyle d^\rho_\pi(s)\!=\! (1\!-\!\gamma) \sum_{i=0}^{\infty} \gamma^i \operatorname{Pr}(S_i=s|S_0\sim \rho, A_j \sim \pi({s_j}),\forall j\leq i)$. We use the shorthand notation $\textstyle \langle\cdot, \cdot\rangle$---the dot product, $\textstyle \nabla f(x) \doteq\nabla_x f(x)$---gradients and partial derivatives, $\textstyle \nabla f(x,y) \doteq\nabla_x f(x,y)$, $\pi_t \doteq \pi_{\theta_t}$, $Q_t \doteq Q_{\pi_t}$, $V_t \doteq V_{\pi_t}$, $d^{\rho}_t \doteq d^{\rho}_{\pi_t}$, $\forall t$.

\textbf{PG Algorithms} update the parameters $\theta \in \Theta$ of a parametric policy $\pi_{\theta}$ using surrogate objectives that are local approximations of the original performance. In the tabular setting, the direct parametrization associates a parameter to each state-action pair, allowing the shorthand notation $\pi \doteq \pi_{\theta}$. The gradient of the performance $V^\rho_{\pi}$ with respect to the \emph{direct/functional representation $\pi$} \citep{Sutton2000, Agarwal19, BhandariandRusso_globalguarantees}---which we call the \textbf{\emph{``functional'' gradient}} (to distinguish it from the gradient with respect to the policy parameter $\nabla_{\theta} V^\rho_{\pi_{\theta}}$), is 
$\nabla_{\pi({a|s})} V_\rho^\pi=\nicefrac{1}{(1-\gamma)} d^\rho_{\pi}(s) Q_{\pi}({s,a}) \in \R^{|\A|}$. 
Then, we define $\nabla V^{\rho}_\pi \!\in\! \R^{|\S|\!\times\!|\A|}$ as the concatenation of $\nabla_{\pi({a|s})} V^{\rho}_\pi, \forall s\!\in\!\S$ (yielding a PGT for directional derivatives--- Lemma~\ref{Policy Gradient Theorem for Directional Derivatives} in Appendix~\ref{append:proofs-and-derivations functional policy gradient}).

\textbf{Mirror Descent (MD)} is a general GD algorithm, applicable to constrained spaces $\C$, which relies on Fenchel conjugate duality to map the iterates of an optimization problem $x^*= \argmin_{x\in\X\cap\mathcal{C}} f(X)$, back and forth between a primal $\X$ and a dual space $\X^*$. The algorithm uses a convex function of the Legendre-type (convex and smooth), called a mirror map $h$, to translate the MD iterates $x$ to the dual space where the gradient update is performed $\nabla h(y) \doteq \nabla h(x) - \eta \nabla f(x)$, with $\eta$ a step size. A new iterate satisfying the primal constraints is obtained using a Bregman projection $x^\prime \doteq \operatorname{proj}^h_{\C}(y) = \argmin_{x\in \C} D_{h}(x, \nabla h^*(\nabla h(y)))$ of the updated dual iterate $\nabla h(y)$ mapped back in the primal space using the Fenchel conjugate function of $h$, $h^*(x^*) = \sup_{x\in\X} \langle x, x^*\rangle - h(x)$. This projection relies on a Bregman divergence  $D_h(x,y) \doteq h(x) - h(y) - \langle\nabla h(y), x - y\rangle$ \citep{Amari1998, Bubeck15, Banerjee2005}. 
The proximal formulation of MD merges the update and projection steps to $x^\prime \doteq \argmin_{\bar{x}\in \X\cap\C} \eta \langle \nabla f(x), \bar{x} \rangle + D_h(\bar{x}, x)$ (Lemma~\ref{lemma:Proximal perspective on mirror descent} in Appendix~\ref{appendix:proofs_and_derivations:PMD}).

\textbf{PMD} is an instance of MD \citep{BeckT03}, applying GD in a non-Euclidean geometry, using the proximal perspective of MD, 
$\pi_{t+1} \doteq \argmin_{\pi \in \Pi}
-  \langle  \nabla V^\rho_{t}, \pi \rangle +\eta^{\!-\!1}_{\pi_{t+1}} D_{{h}}(\pi, \pi_t)$ for some sequence of step-sizes $\eta_{\pi_{t+1}} \!> \!0$ and initial policy $\pi_0$. 
The visitation-distribution $d^\rho_t$ in the gradient of the surrogate objective can lead to vanishing gradients in infrequently visited states under $\pi_t$ \citep{mei20b, bhandari21a, johnson2023optimal}, so PMD iteratively applies a variant that separates the objective per state
\eq{
\textstyle\pi_{t+1}(s) \doteq \argmin_{\pi(s) \in \Delta(\A)} - \langle   {Q}_{t}(s), \pi(s) \rangle + \eta_{\pi_{t+1}}^{\!-\!1}D_{{h}}(\pi(s), \pi_t(s))
}
Inexact PMD replaces ${Q}$ with a value estimator $\widehat{Q}$.
Using the negative Boltzmann-Shannon entropy \citep{shannon1948mathematical} as mirror map yields the Natural Policy Gradient (NPG) \citep{kakade2001}. With a null Bregman divergence, it recovers PI.

\textbf{Approximate PMD}\quad The standard PMD algorithm is adapted by \citet{Tomar20} to general policy parametrization by updating the parameters of $\pi_{\theta}$ using a surrogate composite objective $\textstyle\theta_{t+1} = \argmin_{\theta\in\Theta} \E_{s\sim d^\rho_t}[-\E_{a\sim\pi_\theta(s)}[{Q}_t({s,a})] + \eta_{\pi_{t+1}}^{\!-\!1} D_{h}(\pi_{\theta}(s), \pi_{t}(s))]$.
\citet{alfano2024novel} introduces the concept of Bregman policy class $\{\pi_{\theta}:\pi_{\theta}(s) = \operatorname{proj}^h_{\Delta(\A)}(\nabla h^*(f_\theta(s))), s \in \S\}$, and uses a parametrized function $f_{\theta}$ to approximate the dual update of MD $f_{t+1}(s) \doteq \nabla h(\pi_t(s)) - \eta_{\pi_{t+1}} \widehat{Q}_t(s)$. To satisfy the simplex constraint, \citep{xiong2024dualapproximationpolicyoptimization} uses a Bregman projection on the dual approximation mapped back to the policy space 
$\pi_\theta(s) = \operatorname{proj}^{h}_{\Delta(\A)} (\nabla h^*(f_{t+1}(s)))$, equivalent to 
$\theta_{t+1} = \argmin_{\theta \in \Theta}
D_{h}(\pi_{\theta}(s), \nabla h^*(f_{t+1}(s)))$. 
 Using the negative Boltzmann-Shannon entropy, yields the softmax policy class
 ${\pi}_{\theta}({s,a}) \doteq \nicefrac{\exp f_\theta({s,a})}{\|\exp f\theta({s})\|_1}, \forall s,a\in\S\times\A$. 
 For the approximate setting, we rely on the proximal formulation by \citet{Tomar20}. Additional details on other are provided in Appendix~\ref{append:generalizedGD}.
 
\section{Functional Acceleration for PMD}\label{sec:Functional Acceleration for PMD} 

\paragraph{Extrapolation from the future} Consider first an idealized update, denoted \textbf{\pmdloo}, anticipating one iteration ahead on the optimization path. Given ${\pi}_t$ and $Q_t$, we define the expected lookahead return of acting greedily for one iteration, and following $\pi_t$ thereafter, $\forall s\in\S, a\in \A$ 
\eqq{
\mathcal{T}_{\mu_t}{Q}_t(s,a) \doteq \E_{s^\prime\sim P(\cdot|s, a)}[r({s,a}) + \gamma \E_{a^\prime\sim\mu_t({s^\prime})} 
[{Q}_t({s^\prime, a^\prime})]], \ \text{where}\
\mu_t(s) \doteq \operatorname{greedy}{Q}_t(s)
 \label{eq:greedy_loo}
}
 Next, $\forall s\in\S$ we define
 the proximal update 
\eqq{
\pi_{t+1}(s) &\doteq  \argmin_{\pi(s)\in\Delta(\A)}  \!-\!\langle (\mathcal{T}_{\mu_t}{Q}_t)(s), \pi(s)\rangle + {{\eta}^{\!-\!1}_{\pi_{t+1}}} D_h(\pi(s), \pi_t(s)) \label{eq:bw_loo}
}
where $\eta_{\pi_{t+1}}$ is a step size.
Thm.~\ref{Functional acceleration with pmdloo} indicates that replacing the \PMD\  update with the two updates of \pmdloo\ (Eq.~\ref{eq:greedy_loo} \& Eq.~\ref{eq:bw_loo}) changes the convergence rate of the algorithm, featuring $\gamma^2$ in place of $\gamma$ for \PMD\footnote{$\gamma^2$ corresponds to the one-step lookahead horizon $H\!=\!1$, and generalization to multi-step lookahead would yield $\gamma^{H+1}$}. 


\begin{theorem}{(Functional acceleration with exact {\pmdloo})}
\label{Functional acceleration with pmdloo}
The  policy iterates $\pi_{t+1}$ of \emph{\pmdloo}\
satisfy
$\|V^{*}\! -\!V_{t}\|_{\infty} \!\leq\! (\gamma^{2})^t (\| V^{*}  \!-\! V_{0}\|_{\infty} + \sum_{i\leq t} {(\gamma^2)^i}{\eps^{\!-\!1}_{\pi_{i+1}}})$,
if ${\eta}_{\pi_{t+1}}$ is an adaptive step-size such that $\eta^{\!-\!1}_{\pi_{t+1}}D_{h}(\operatorname{greedy}(\mathcal{T}_{\mu_t} Q_t)(s),\pi_t(s))\leq \eps_{\pi_{t+1}}$, $\forall \eps_{\pi_{t+1}}$ arbitrarily small. 

(Proof. in Thm.~\ref{append:Functional acceleration with PMDloo}, Appendix~\ref{append:Proofs for Sec. Functional Acceleration for PMD}) 
\end{theorem}

It is known that the hard greedification operator in the update of $\mu_t$ can yield unstable updates if we replace ${Q}_t$ with a value estimator $\widehat{Q}_t$ \citep{russo2022approximation} (see also Sec.\ref{sec:Acceleration with an inexact critic}, Fig.~\ref{fig:2state_learning_curves}). 
Consequently, 
 we relax the update of \pmdloo\ by replacing the hard greedification in Eq.~\ref{eq:greedy_loo} with a PMD update
\eqq{
\mu_{t}(s) &\doteq  \argmin_{\pi(s)\in\Delta(\A)}  \!-\! \langle {Q}_{t}(s), \pi(s)\rangle + {\eta^{\!-\!1}_{\mu_t}}D_h(\pi(s), \pi_t(s)) \label{eq:fw_pmd_ext}
}
where ${\eta^{\!-\!1}_{\mu_t}}$ is a step-size. We refer to Eq.~\ref{eq:fw_pmd_ext} \& ~\ref{eq:bw_loo} as \textbf{\pmdext}\ (cf. `extragradient'' or extrapolated gradient method of \citet{Nemirovski04}). Thm.~\ref{Functional acceleration with pmdext} confirms acceleration  is maintained.
\begin{proposition}{(Functional acceleration with {\pmdext})}
\label{Functional acceleration with pmdext}
The policy iterates $\pi_{t+1}$ of {\pmdext}\ satisfy
$\|V^{*}\! -\!V_{{t}}\|_{\infty} \!\leq\! (\gamma^{2})^{t} (\| V^{*}  \!-\! V_{0}\|_{\infty} \!+\! \sum_{i\leq t}\! \nicefrac{(\eps_{\pi_{t+1}} \!+ \gamma\eps_{\mu_t})}{(\gamma^2)^i})$, if ${{\eta}_{\mu_t}}, \eta_{\pi_{t+1}}$ are adaptive step-sizes such that 
 ${{\eta}^{\!-\!1}_{\mu_t}}D_{h}(\operatorname{greedy}({Q}_{t}(s)), \pi_t(s))\leq \eps_{\mu_t}$, and $\eta^{\!-\!1}_{\pi_{t+1}}D_{h}(\operatorname{greedy}(\mathcal{T}_{\mu_t} Q_t)(s),\pi_t(s)\leq \eps_{\pi_{t+1}}$, respectively, $\forall {\eps}_{\mu_t},\eps_{\pi_{t+1}}$ arbitrarily small.
\end{proposition}
(Proof. in Thm.~\ref{append:Functional acceleration with PMDext}, Appendix~\ref{append:Proofs for Sec. Functional Acceleration for PMD}) 

\paragraph{Limitations}
\textbf{First}, both \pmdloo\ and {\pmdext}\  update two policies  per iteration, both from from $\pi_t$ ($\mu_t$ to look ahead and $\pi_{t+1}$ as the next policy iterate). \textbf{Second}, they require two policy evaluation procedures per iteration, for $Q_{t}(s)$ and $(\mathcal{T}_{\mu_t} Q_t)(s)$. We may use a value estimator for the former, but the latter would then need access to a model. The main issue of \pmdext\ is that it needs information from two consecutive updates ($\pi_{t}$ and $\mu_t$) to compute the update to the policy iterate $\pi_{t+1}$. In the next section, we propose a ``lazy'' practical algorithm which remedies these issues, using only one policy variable and applying a single policy evaluation procedure per iteration, at the expense of extra memory. 

\paragraph{Extrapolation from the past} 
 The idea is to rewrite the update to use a single  evaluation per timestep plus a temporal-difference error between consecutive evaluation updates. Then, assuming past errors reflect future ones, reusing the prior errors (extrapolating from the past) should have the same effect as extrapolating from the future.
 We briefly touch on the main steps of algorithm derivation, deferring details to Appendix~\ref{Extrapolation from the past: derivation of the lazy momentum-based PMD algorithm}. \textbf{First}, we show the \pmdext\ update can be rewritten as
$\textstyle \pi_{t+1}(s) =  \argmin_{\pi(s)\in\Delta(\A)} \!-\! \langle (\mathcal{T}_{\mu_t}{Q}_{t})(s)  \!-\! \eta^{\!-\!1}_{\pi_{t+1}} {\eta}_{\mu_t} {Q}_t(s), \pi(s)\rangle +{\eta}^{\!-\!1}_{\mu_t
}D_h(\pi(s), \mu_t(s))$.
\textbf{Next}, to remedy the need for two evaluation procedures, we recycle evaluation updates from previous iterations by replacing $Q_{t}$ with the closest past evaluation, i.e. $\mathcal{T}_{\mu_{t-1}} Q_{t-1}$, twice in the updates of $\pi_{t+1}(s)$ and $\mu_t$ (Eq.~\ref{eq:fw_pmd_ext}).
We \textbf{then} show the update can be written in terms of a one policy variable and a single evaluation procedure, henceforth denoted $\pi$ and ${Q}$
\eq{
\pi_{t+1}(s) &\!=\! 
\argmin_{\pi(s)\in\Delta(\A)}
\!-\! \langle {Q}_{t}(s)  \!+\! \eta^{\!-\!1}_{\pi_{t+1}}\eta_{\pi_{t}}({Q}_{t}(s)  \!-\! {Q}_{t-1}(s)), \pi(s)\rangle \!+\! \eta^{\!-\!1}_{\pi_{t+1}}D_h(\pi(s), \pi_t(s)) 
}
We refer to the algorithm using this update \textbf{\pmdmom}, since the temporal-difference is analogous to momentum. Interestingly, a similar update is called ``optimistic'' mirror descent by \citet{joulani20a, JoulaniGS20, RakhlinS13, rakhlin2014online} and a ``forward-reflected-backward'' method by \citet{malitsky2020forwardbackward}. 

Alg.~\ref{alg:PMD++} in Appendix~\ref{append:Algorithms} summarizes the updates of all algorithms introduced.

\subsection{Approximate Functional Acceleration for Parametric Policies}\label{sec:Approximate Functional Acceleration for Parametric Policies}
We are interested in designing algorithms feasible for large-scale optimization, so we further consider parametrized versions of the algorithms introduced, which we illustrate numerically in Sec.~\ref{sec:Numerical studies}.

\textbf{Q-function Approximation}\quad For the \emph{exact setting}, we compute model-based versions of all updates. For the \emph{inexact setting}, we consider approximation errors between $\widehat{Q}_{t}$ and $Q_{t}$ (Sec.\ref{sec:Acceleration with an inexact critic}).

\begin{algorithm}[H]
\caption{{\textbf{Approximate \pmdmom}}}
\label{alg:PMD+}
\begin{algorithmic}[1]
{\footnotesize
\STATE Initialize policy parameter 
$\theta_0 \in \Theta$, mirror map $h$, small constant $\eps_0$, learning rate $\beta$
  \FOR{$t = 1,2 \dots T$}
\STATE Find $\widehat{Q}_t$ approximating $Q_{t}$ (critic update)
\STATE Compute adaptive step-size $\eta_{\pi_{t+1}} = (\gamma^{2(t+1)}\eps_0)^{\!-\!1}{D_{h}(\operatorname{greedy}(\widehat{Q}_t), \pi_t)}$
\STATE  Find $\pi_{t+1} \doteq \pi_{\theta_{t+1}}$ by solving the surrogate problem (approximately with $k$ GD updates)
\STATE$\textstyle\operatorname{min}_{\theta\in\Theta}\ell(\theta)\quad \ell(\theta)\!\doteq \!-\!\E_{s\sim d^{\rho}_{t}}[\E_{a\sim \pi_{\theta}}[ {Q}_t({s,a})]\! +\!{{\eta}^{\!-\!1}_{\pi_{t\!+\!1}}} D_{h}(\pi_\theta(s),{\pi}_t(s))]$
\STATE  \ $\text{\emph{(init)}}\ \theta^{(0)} \doteq \theta_t \quad \text{\emph{(for $i \in [0..k\!-\!1]$)}}\ \theta^{(i+1)} = \theta^{(i)} - \beta \nabla_{\theta^{(i)}} \ell(\theta^{(i)})\quad\text{\emph{(final)}}\ \theta_{t+1}  \doteq \theta^{(k)}$
 \ENDFOR
    }
\end{algorithmic}
\end{algorithm}
\vspace{-10pt}
\textbf{Policy Approximation}\quad We parametrize the policy iterates using a Bregman policy class 
$\{\pi_{\theta}:\pi_{\theta}(s) = \operatorname{proj}^h_{\Delta(\A)}(\nabla h^*(f_\theta(s))), s \in \S\}$, a tabular parametrization for the dual policy representation $f_{\theta}({s,a}) \doteq \theta({s,a})$, and the negative Boltzmann-Shannon entropy as mirror map $h$, which leads to the softmax policy class ${\pi}_{\theta}(s) \doteq \nicefrac{\exp \theta({s})}{\sum_{a\in\A}\exp \theta({s,a})}$. To update the parameter vector $\theta$ we use the proximal perspective cf. Alg.~\ref{alg:PMD+} \citep{Tomar20, Vaswani2021, vaswani2023decisionaware} (implementation details in Appendix~\ref{append:Algorithmic Implementations}).

\section{Numerical Studies}\label{sec:Numerical studies}
In this section, we investigate numerically the aforementioned algorithms, focusing on the following questions:
Sec.~\ref{sec:When is acceleration possible?}: {When is acceleration possible? What properties of the problem  make it advantageous?} 
Sec.~\ref{sec:Policy dynamics in Value Space}: {How do the policy optimization dynamics change with acceleration? Does the value improvement path change?}
Sec.~\ref{sec:Acceleration with an inexact critic}: {Should we expect acceleration to be effective with an inexact critic? What are the implications of value approximation?}

In all experiments, results indicate mean values and shades standard deviation (std) over seeds.

\subsection{When is Acceleration Possible?}\label{sec:When is acceleration possible?}

\textbf{Experimental Setting}\quad
We consider randomly constructed finite MDPs---\emph{Random MDP} problems 
\citep{Archibald1995}, abstract, yet representative of the kind of MDP encountered in practice, which serve as a test-bench for RL algorithms \citep{goyal2021firstorder, ScherrerGeist2014, Vieillard2019}. We define a \emph{Random MDP} generator  $(|\S|, |\A|, b, \gamma)$ parameterized by number of states $|\S|$, number of actions $|\A|$, branching factor $b$ specifying for each state-action pair the maximum number of possible next states, chosen randomly. We vary $b$, $\gamma$, and $|\A|$ to show how the characteristics of the problem, and the features of the algorithms, impact learning speed with or without acceleration (details in Appendix~\ref{append:Details of Random Markov Decision Processes}). 

\textbf{{Metrics}}\quad We measure the following quantities. (i) The cumulative \emph{regret} (optimality gap) after $T$ iterations, $\emph{Regret}_t \doteq \sum_{t\leq T} V^{*} - V^\rho_{t}$. The relative difference in regret between the \PMD\ baseline and \pmdmom\ shows whether functional acceleration speeds up convergence. To quantify the complexity of the optimization problem and ill-conditioning of the optimization landscape (difference in scaling along dimensions), we use the dual representation form of \citet{Wang2008DualRF} for policies (successor representation \citep{Dayan1993} or state-visitation frequency), $\Psi_{\pi} \doteq (\mathbf{I} - \gamma \mathbf{P}_{\pi})^{-1}$, with $\mathbf{P}_{\pi} V(s) = \E_{a\sim\pi,s^\prime \sim P(\cdot|s,a)}[V({s^\prime})]$. Policy iteration is known to be equivalent to the Newton-Kantorovich iteration procedure applied to the functional equation of dynamic programming \citep{Puterman1979}, $V_{\pi_{t+1}} = V_{\pi_t} - \Psi \nabla f (V_{\pi_t})$, where $\nabla f (V) = (I- \mathcal{T})(V)$---with $\mathcal{T}$ the Bellman operator---can be treated as the gradient operator of an unknown function $f : \R^{|\S|} \to \R$ \citep{grandclément2021convex} (see Appendix~\ref{Newton’s method}). From this perspective, the matrix $\Psi$ can be interpreted as a gradient preconditioner, its inverse is the Hessian $\nabla^2 f(V)$, the Jacobian of a gradient operator $\nabla f$. We use the condition number of this matrix, defined as $\kappa(\Psi) \doteq \nicefrac{|\lambda_{\text{max}}|}{|\lambda_{\text{min}}|}$, for $\lambda_{\text{max}}$, $\lambda_{\text{min}}$ the max and min eigenvalues in the spectrum $\operatorname{spec}(\Psi)$. We measure (ii) the \emph{condition number} $\kappa_0 = \kappa(\Psi_{\pi_0})$ of a randomly initialized (diffusion) policy $\pi_0$ (Fig.~\ref{fig:rmdp_learning_curves}(a-b)) and  (iii) the average \emph{condition number} $\kappa_{t\leq T} = \nicefrac{1}{T} \sum_{t\leq T}\kappa(\Psi_{\pi_t})$, for policies on the optimization path of an algorithm (Fig.~\ref{fig:rmdp_learning_curves}(c)). Lastly, we also measure (iv) the mean \emph{entropy} of a randomly initialized policy $\mathcal{H}_0 \propto \sum_{s,a}\pi_0({s,a}) \log \pi_0({s,a})$ (Fig.~\ref{fig:rmdp_learning_curves}(d)), inversely correlated with $\kappa_0$.

\begin{figure}[h]
\hspace{-10pt} 
        \centering
         \includegraphics[width=1.\textwidth]{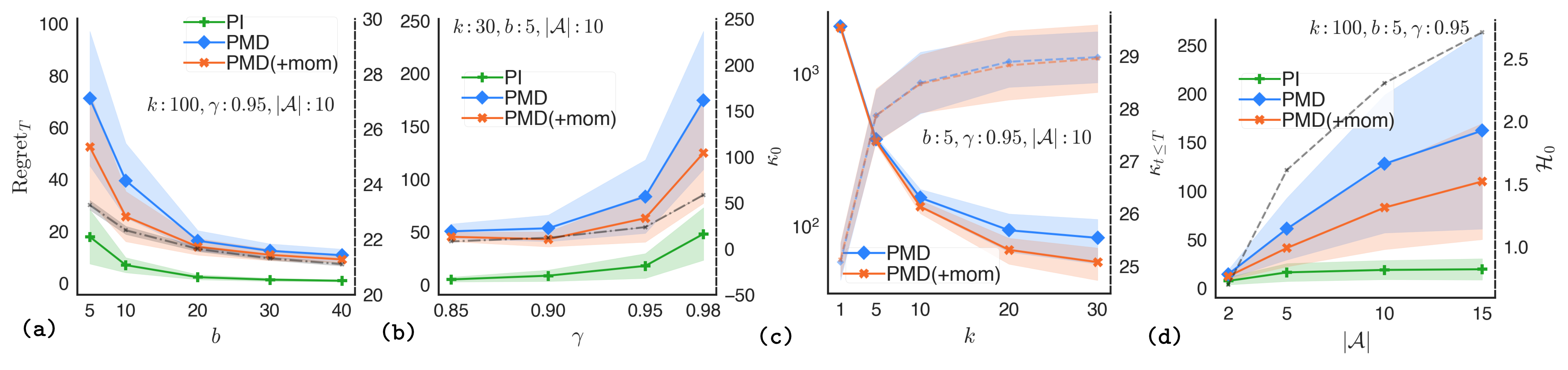}
    \caption{
    The left $y$-axis shows the optimality gap or cumulative regret at $T=10$ \emph{(a-c)}, $T=20$ \emph{(d)}) over $50$ randomly sampled MDPs, relative to changing: (a) $b$---branching factor, (b) $\gamma$---discount factor, (c) $k$---number of parameter updates, (d) $|\mathcal{A}|$---number of actions. The right $y$-axis and dotted curves measure: \emph{(a-b)}---the condition number $\kappa_0$, \emph{(c)} the average condition number $\kappa_{t\leq T}$, \emph{(d)} the entropy $\mathcal{H}_0$. 
    }
    \label{fig:rmdp_learning_curves}
\end{figure}

\textbf{Hypothesis \& Observations}\ 
In Fig.~\ref{fig:rmdp_learning_curves} we show the relative difference in regret between \pmdmom\ and the \PMD\ baseline, as we change the features of the algorithms and the complexity of the problem.
First, we highlight two cases that lead to ill-conditioning---indicated by the condition number $\kappa_0$:
\emph{(a)} sparse connectivity of the underlying Markov chain controlled by decreasing the branching factor $b$, which represents the proportion of next states available at every state-action pair; 
\emph{(b)} increasing the effective horizon via the discount factor $\gamma$. We illustrate the relative difference in optimality gap between the two updates correlates with ill-conditioned policy optimization landscapes, supporting the hypothesis that functional acceleration leads to faster navigation in this case.
In \emph{(c)}, we show the relative difference in regret correlates with the number of parameter updates $k$ used in the ``inner-loop`` optimization procedure. Large $k$ implies the policy iterates approach the analytic solution of the surrogate objective. As $k$ decreases, the added momentum will shrink too, becoming negligible, defaulting to the classic parameter-level momentum in the limiting case of $k\!=\!1$ (the online setting).
In \emph{(d)}, as we increase the number of actions, the optimization problem becomes more challenging, entropy and overall suboptimality increase. However, the relative difference between \pmdmom\ and the baseline \PMD\ also increases, suggesting the increasing advantage of functional acceleration (learning curves and additional statistics in Appendix~\ref{append:Supplementary results for When is acceleration possible?}).

\textbf{Implications} \ These studies indicate (i) that it is possible to accelerate PMD, that the advantage of functional acceleration is proportional to: (iii) the ill-conditioning of the optimization surface, induced by the policy and MDP dynamics, and (ii) the quality of the ``inner-loop'' policy approximation.

\subsection{Policy Dynamics in Value Space}\label{sec:Policy dynamics in Value Space}

To contrast the expected policy dynamics, we rely on the functional mapping from stationary policies to their respective value functions, $\pi \to V_{\pi}$: $\mathcal{V} \doteq \{ V_{\pi} | \pi \in \Pi\}$. The space of value functions $\mathcal{V}$ the set of all value functions that are attained by some policy, i.e. the image of $\Pi$ under the mapping, and a possibly self-intersecting, non-convex polytope \citep{dadashi2019value}. 

\begin{figure}[h]
        \centering
         \hspace{-20pt} 
         \includegraphics[width=1.04\textwidth]{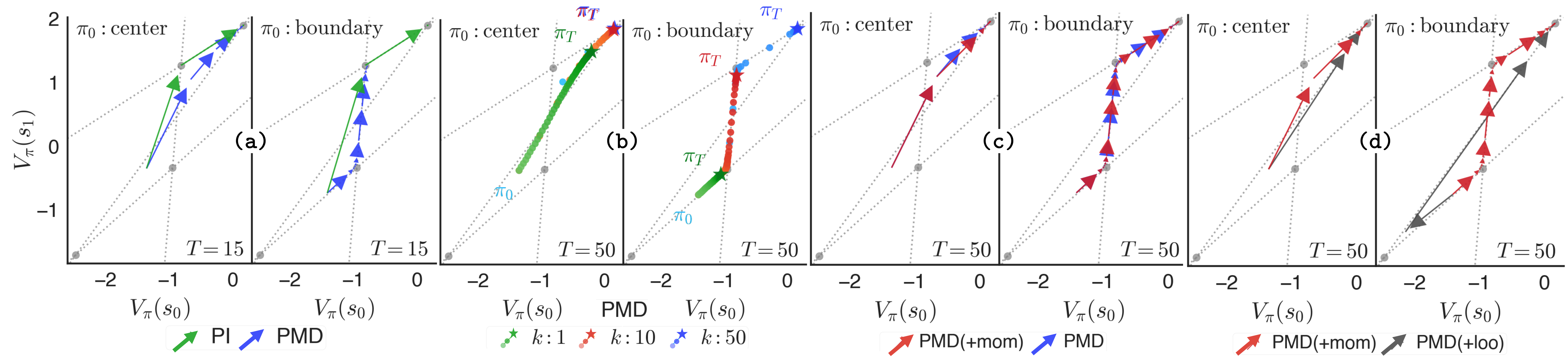}
    \caption{
    Compares and contrasts the policy optimization dynamics on the value polytope. {\color{gray}Gray points} denote the boundaries---corresponding in this case to value functions, {\color{gray} gray dotted lines} are hyperplanes circumscribing the polytope. Top-left: policy initialization. \emph{(a, c, d)} Arrows indicate the policy optimization path. {\color{pmdcolor}\PMD}\ and {\color{pmdmomcolor}\pmdmom}\ use $\beta=0.1$ and $k=50$ for the ``inner'' loop optimization, approximating the analytic solution.
    \emph{(b)} Points denote policy values on the path, gradient indicates iteration number $t$, star $\star$ marks $\pi^T$. 
    }
    \label{fig:pi_pmd_loo_dynamics}
\end{figure} 
\textbf{Experimental Setting}\quad
We use two-state MDPs (specifics in Appendix~\ref{append:Details of two-state/action Markov Decision Processes}, other MDPs in Appendix~\ref{append:Supplementary results for Policy dynamics in Value Space}). Policies are initialized with $\pi_0$: (i) \emph{center}: interior of the polytope, (ii) \emph{boundary}: near a boundary close to the corner adversarial to the optimum. 
We use the value polytope to visualize three aspects of the learning dynamics: (1) the policy improvement path through the polytope, (2) the speed at which they traverse the polytope, and (3) sub-optimal attractors with long escape times that occur along this path, making the policy iterates accumulate (cf. \cite{mei20b}). We compute model-based versions of all relevant updates. 

\textbf{Observations \& Insights}\quad
Fig.~\ref{fig:pi_pmd_loo_dynamics} shows the impact of policy approximation through $k$ \emph{(a-b)}, and compares the dynamics of \pmdmom\ relative to the baselines: without acceleration: \PMD\ \emph{(c)}, and with idealized acceleration: \pmdloo\ \emph{(d)}.

We make the following observations: 
\emph{(a)} \PMD's dynamics follow a straight path between the iterates of \PI, consistent with the former being an approximation of the latter. The convergence speed depends on the approximation quality through $k$---the number of parameter updates per iteration \emph{(b)}.
\PMD\ with $k\!=\!1$ corresponds to online PG, which depends strongly on initialization, a known issue caused by vanishing gradients at the boundary of the polytope.
In contrast, as we increase $k$, there is faster escape from sub-optimal attractors, and the rate of convergence is higher, becoming more similar to \PI\ by jumping between values of deterministic policies (corners of the polytope \emph{(a)}).
\emph{(d)} \pmdloo\ follows a different trajectory through the value polytope.
We observe acceleration for \pmdmom\ relative to \PMD\ in \emph{(c)}, and suboptimality relative to the idealized acceleration of \pmdloo\ in \emph{(d)}.

\subsection{Functional Acceleration with an Inexact Critic}\label{sec:Acceleration with an inexact critic}
For the same experimental setting as Sec.~\ref{sec:Policy dynamics in Value Space}, Fig.~\ref{fig:2state_learning_curves} illustrates the impact of an inexact critic on the relative advantage of functional acceleration, in two settings: \emph{(Left)} \emph{controlled}---the critic's error is sampled from a random normal distribution with mean $0$ and standard deviation $\tau$, such that ${\widehat{Q}_t(s)}=Q_t(s) + \mathcal{N}(0, \tau)$, $\forall s$. \emph{(Right)} \emph{natural}---the critic is an empirical estimate of the return obtained by Monte-Carlo sampling, and its error arises naturally from using $m$ truncated trajectories up to horizon $\nicefrac{1}{1-\gamma}$, i.e. 
${\widehat{Q}_t(s)} \doteq \nicefrac{1}{m}\sum_{i\leq m} \nicefrac{G_i(s)}{N_i(s)}$, where $G_i(s)$ is the $i^{\text{th}}$ empirical return sampled with ${\color{bblue} {\pi}_t(s)}$ and $N_i(s)$ is the empirical visitation frequency of $s$.

\begin{figure}[h]
        \centering
         \vspace{-1pt} 
         \hspace{-10pt} 
          \includegraphics[width=1.\textwidth]{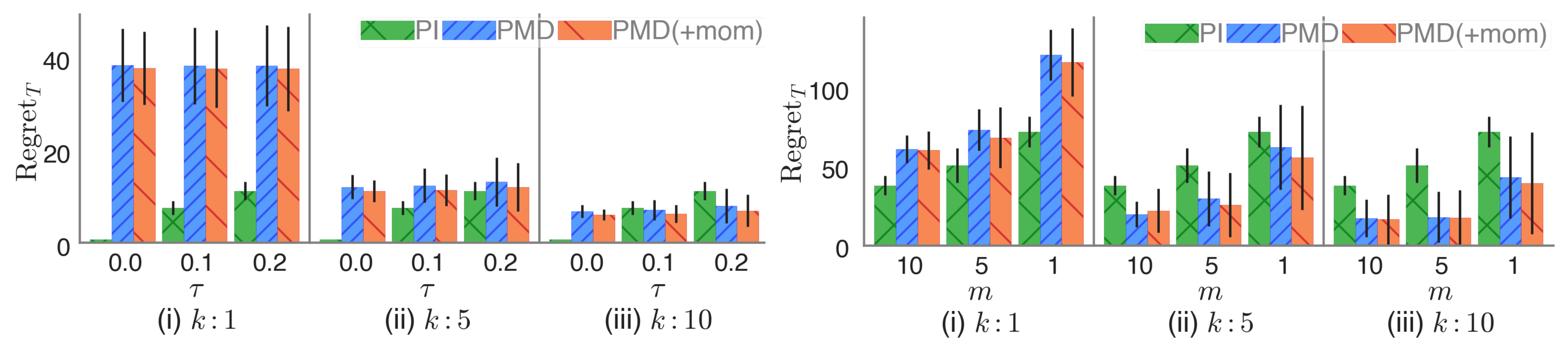}
              \caption{
    Shows cumulative regret on the $y$-axis, after $T=50$ iterations, relative to changing $k$---the number of parameter updates for {\color{pmdcolor}\PMD}\ and {\color{pmdmomcolor}\pmdmom}, in the inexact setting. \emph{(Left) controlled}: $\tau$ indicates the scale of the critic's error. \emph{(Right) natural}: $m$ is the number of trajectories used in the Monte-Carlo estimation of the return. Error bars denote std. over $50$ seeds with policies initialized from a random uniform distribution $\mathcal{U}(0,1)$.
    }
    \label{fig:2state_learning_curves}
\end{figure}

We observe a larger relative  difference  in suboptimality on average between \pmdmom\ and \PMD\ for higher values of $k$,  highlighting the difference between functional acceleration (cf. Sec.\ref{sec:Functional Acceleration for PMD}) and classic acceleration (applied to the parameter vector $\theta$), corresponding to $k\!=\!1$, reinforcing evidence from Sec.~\ref{sec:When is acceleration possible?}.
Further, we confirm \PI\ performs increasingly poor when paired with an inexact critic with large error as previously reported by \citep{russo2022approximation}. 
Then, we observe a range in which functional acceleration is advantageous, which extends from having negligible benefit, for small $k$, to more impactful differences in optimality gap for larger $k$. Beyond a certain sweet spot, when it is maximally advantageous, the critic's error becomes too large, leading to overshooting and oscillations (additional illustrations of this phenomenon in Appendix~\ref{append:Supplementary results for Acceleration with an inexact critic}). 

\section{Closing}\label{Closing}
Inspired by functional acceleration from convex optimization theory, we proposed a momentum-based PMD update applicable to general policy parametrization and large-scale optimization. We analyzed several design choices in ablation studies designed to characterize qualitatively the properties of the update, and illustrated numerically how the characteristics of the problem influence the added benefit of using acceleration. Finally we looked at how inexact critics impact the method. Further analysis with these methods using stochastic simulation and function approximation would be useful.



\subsubsection*{Acknowledgments}
\label{sec:ack}
The authors thank Jincheng Mei, Hado van Hasselt and all reviewers for feedback and insights. Veronica Chelu is grateful for support from IVADO, Fonds d’excellence en recherche Apogée Canada, Bourse d’excellence au doctorat. 


\bibliography{main}
\bibliographystyle{rlj}

\clearpage
\beginSupplementaryMaterials


\appendix
\section{Notation}
\begin{table}[h]
\caption{Notation}
\label{table:notation_table}
\footnotesize
\begin{tabularx}{\textwidth}{p{0.1\textwidth}X}
\toprule
$t$ &  iteration number
\\
$T$ &  max number of iterations
\\
$k$ &  number of GD updates for the ``inner-loop'' proximal optimization procedure
\\
$\eta_\mu$, $\eta_{
\pi}$ &  step sizes for the proximal update (regularization strength of the divergence)
\\
$\beta$ &  step size for the ``inner-loop'' parameter-level optimization procedure
\\
$h$ & the mirror map 
\\
$D_h(\pi, \mu)$ & Bregman divergence associated with the mirror map $h$
 \\\bottomrule
 \end{tabularx}
\end{table}

\section{Proofs and derivations}
\label{apend:proofs_and_derivations}

\subsection{Proofs and Derivations for Sec.\ref{sec:background_and_notation}: Background \& Preliminaries}
\label{appendix:proofs_and_derivations:PMD}

\subsubsection{Functional Policy Gradient}\label{append:proofs-and-derivations functional policy gradient}
The Performance Difference Lemma (PDL) is a property that relates the difference in values of policies to the policies themselves.
\begin{lemma}{\textbf{(Performance Difference Lemma from \citet{cpi})}}
\label{Performance difference lemma} For any policies $\pi_{t+1}$ and $\pi_t$, and an initial distribution $\rho$
\eq{
V^\rho_{t+1}-V^\rho_t 
& =\nicefrac{1}{1-\gamma} \sum_{s\in\S} \sum_{a \in \mathcal{A}} d^\rho_{{t+1}}({s})(\pi_{t+1}({s,a})-\pi_t({s,a}), Q_t({s,a}))
\\
&= \nicefrac{1}{1-\gamma}\E_{s\sim d^{\rho}_{t+1}}\left[\langle Q_t(s),\pi_{t+1}(s)- {\pi}_t(s)\rangle\right]
}
\end{lemma}
\begin{proof} 
According to the definition of the value function
\eq{
V_{t+1}(s)-V_t(s) & =\langle  Q_{t+1}(s),\pi_{t+1}(s)\rangle- \langle Q_t(s),\pi_t(s)\rangle
\\
&= \langle Q_t(s), \pi_{t+1}(s)-\pi_t(s)\rangle+\langle Q_{t+1}(s)-Q_t(s),\pi_{t+1}(s)\rangle
\\
&=  \langle Q_t(s), \pi_{t+1}(s)-\pi_t(s)\rangle
+ \gamma  \sum_{s^{\prime}} \sum_a  P(s^{\prime} | s, a) \pi_{t+1}({s,a})  [V_{t+1}({s^{\prime}})-V_{t}({s^{\prime}})]
\\
&=\nicefrac{1}{1-\gamma} \sum_{s^{\prime}} d_{t+1}({s\to s^{\prime}}) \langle Q_{t}({s^{\prime}}), \pi_{t+1}({s^{\prime}})-\pi_t({s^{\prime}})\rangle 
}
\end{proof} 

The following lemma is a version of the policy gradient theorem \citep{Sutton2000} applied to the direct policy representation---the functional representation of the policy probabilities, and has appeared in various forms in \citet{Agarwal19, BhandariandRusso_globalguarantees, bhandari21a, russo2022approximation}.
\begin{lemma}{\textbf{(Policy Gradient Theorem for Directional Derivatives)}}\label{Policy Gradient Theorem for Directional Derivatives}
For two policies $\pi_{t+1}, {\pi}_t \in \Pi$
\eq{
\langle\nabla V^\rho_{t}, \pi_{t+1} - {\pi}_t\rangle&=\sum_{s \in \mathcal{S}} \sum_{a \in \mathcal{A}} d_t(s) Q_t({s,a})(\pi_{t+1}({s, a})-{\pi}_t({s, a}))
\\
&= \E_{s\sim d_t}\left[\langle Q_t(s),\pi_{t+1}(s)- {\pi}_t(s)\rangle\right]
}
\end{lemma}
\begin{proof} 
A Taylor expansion using the Performance Difference Lemma~\ref{Performance difference lemma} reveals
\eq{
V^\rho_{t+1}-V^\rho_t 
& =\nicefrac{1}{1-\gamma} \sum_{s\in\S} \sum_{a \in \mathcal{A}} d^\rho_{{t+1}}({s})(\pi_{t+1}({s,a})-\pi_t({s,a}), Q_t({s,a}))
\\
&= \E_{s\sim d^{\rho}_{t+1}}\left[\langle Q_t(s),\pi_{t+1}(s)- {\pi}_t(s)\rangle\right]
\\
&= \nicefrac{1}{1-\gamma} \E_{s\sim d^\rho_t}\left[\langle Q_t(s),\pi_{t+1}(s)- {\pi}_t(s)\rangle\right] \\
&\qquad+ \nicefrac{1}{1-\gamma}\underbrace{\sum_{s \in \mathcal{S}} 
(d^\rho_{{t+1}}({s})-d^\rho_t({s}))\langle Q_t(s),\pi_{t+1}(s)- {\pi}_t(s)\rangle}_{=O(\|\pi_{t+1}(s)- {\pi}_t(s)\|^2)}
}
The last error term is second-order since
$P_\pi$ is linear in $\pi$ and then $d^\rho_\pi$
is differentiable in $\pi$.
\end{proof}

The next lemma states that the MD method minimizes the local linearization of a function while not moving too far away from the previous point, with distances measured via the Bregman divergence of the mirror map.
\begin{lemma}
{\textbf{(Proximal perspective on mirror descent)}}\label{lemma:Proximal perspective on mirror descent} The MD update for $x\in\mathcal{X} \cap \mathcal{C}$, with mirror map $h : \mathcal{X} \to \R$ for the minimization problem $\min_{x\in\mathcal{X} \cap \mathcal{C}} f(x)$, with $f:\mathcal{X}\to \R$ can be rewritten in the following ways, for step-size $\eta\geq 0$ and $t\geq 0$
\eq{
x_{t+1} &= \argmin_{x\in\mathcal{X}\cap \mathcal{C}}D_h(x,\nabla h^*(\nabla h(x_t) + \eta \nabla f(x_t)))&&\text{(PGD)}
\\
 &= \argmin_{x\in\mathcal{X}\cap \mathcal{C}} \eta \left\langle \nabla f(x_t), x\right\rangle + D_h(x, x_t) &&\text{(proximal perspective)}
}
\end{lemma}
\begin{proof} 
\eq{
x_{t+1} &= \argmin_{x\in\mathcal{X}\cap \mathcal{C}} D_h(x,\nabla h^*(\nabla h(x_t) + \eta \nabla f(x_t)))&&\text{(generalized GD)}
\\
&=\argmin_{x\in\mathcal{X}\cap \mathcal{C}}  h(x) - \left\langle\nabla h(\nabla h^*(\nabla h(x_t) + \eta \nabla f(x_t))), x\right\rangle
\\
&=\argmin_{x\in\mathcal{X}\cap \mathcal{C}}  h(x) - \left\langle\nabla h(x_t) + \eta \nabla f(x_t), x\right\rangle
\\
&= \argmin_{x\in\mathcal{X}\cap \mathcal{C}} \eta \langle \nabla f(x_t), x\rangle + D_h(x, x_t) &&\text{(proximal perspective)}
}
\end{proof}

\subsubsection{Helpful Lemmas for Policy Mirror Descent}

Key to the analysis of \citet{xiao2022convergence} and \citet{johnson2023optimal} is the Three-Point Descent Lemma, that relates the improvement of the proximal gradient update compared to an arbitrary point. It originally comes from \citet{Chen1993} (Lemma 3.2).

\begin{lemma}{\textbf{(Three-Point Descent Lemma, Lemma 6 in \citet{xiao2022convergence})}}\label{Three-Point Descent Lemma}. Suppose that $\mathcal{X} \subset \mathbb{R}^n$ is a closed convex set, $\psi: \mathcal{X} \rightarrow \mathbb{R}$ is a proper, closed convex function, $D_\Omega(\cdot, \cdot)$ is the Bregman divergence generated by a function $\Omega$ of Legendre type and $\operatorname{rint} \operatorname{dom}\Omega \cap \mathcal{X} \neq \emptyset$. For any $x_t \in \operatorname{rint} \operatorname{dom}\Omega$, let
\eq{
x_{t+1}=\operatorname{argmin}_{x \in \mathcal{X}}\psi(x)+D_\Omega(x, x_t)\nonumber
}
Then $x_{t+1} \in \operatorname{rint} \operatorname{dom} \Omega \cap \mathcal{X}$ and $\forall x \in \mathcal{X}$,
\eq{
\psi(x_{t+1})+D_\Omega(x_{t+1}, x_t) \leq \psi(x)+D_\Omega(x, x_t)
-D_\Omega(x, x_{t+1})
}
The PMD update is an instance of the proximal minimisation with $\mathcal{X}=\Delta(\mathcal{A}), x_t=\pi_t(s)$ and $\psi(x)=-\eta_{\pi_{t+1}}\langle Q_t(s), x\rangle$. Plugging these in, the Three-Point Descent Lemma relates the decrease in the proximal objective of $\pi_{t+1}(s)$ to any other policy, i.e. $\forall \pi(s) \in \Delta(\mathcal{A})$,
\eq{
-\eta_{\pi_{t+1}} \langle  Q_t(s), \pi_{t+1}(s) 
 - \pi(s)\rangle\leq
D_\Omega(\pi(s), {\pi}_{t}(s))
- D_\Omega(\pi_{t+1}(s), \pi_t(s))
 -D_\Omega(\pi(s), {\pi}_{t+1}(s))  \label{eq:exact PMD}
}
\end{lemma}
This equation is key to the analysis of convergence of exact PMD, leading to Lemma~\ref{Descent Property of PMD} regarding the monotonic improvement in Q-functions of PMD iterates.

\begin{lemma}{\textbf{(Descent Property of PMD for Q-functions, 
Lemma 7 in \citet{xiao2022convergence} 
}}\label{Descent Property of PMD_interm} 
Consider the policies produced by the iterative updates of exact PMD. For any $t \geq 0$
\eq{
\langle Q_t(s), \pi_{t+1}(s)-\pi_t(s)\rangle \geq 0, \quad \forall s \in \mathcal{S}, \nonumber
}
\end{lemma}
\begin{proof}
From the Three-Point Descent Lemma~\ref{Three-Point Descent Lemma} of \citet{xiao2022convergence} with $\pi=\pi_t$,
\eq{
  \eta_{\pi_{t+1}}\langle Q_t(s), \pi_{t+1}(s)-\pi_t(s)\rangle\geq D_\Omega(\pi_t(s), \pi_{t+1}(s))+D_\Omega(\pi_{t+1}(s), \pi_t(s)) 
}
since the Bregman divergences are non-negative and $\eta_{\pi_{t+1}}>0$,
\eq{
\langle Q_t(s), \pi_{t+1}(s)-\pi_t(s)\rangle\geq 0
}
\end{proof}

\begin{lemma}{\textbf{(Descent Property of PMD for Value Functions, Lemma A.2. from \citet{johnson2023optimal})}}\label{Descent Property of PMD} 
Consider the policies produced by the iterative updates of exact PMD. Then for any $t \geq 0$,
\eq{
Q_{t+1}(s) &\geq Q_t(s), \quad \forall s\in \mathcal{S}
\\
V^\rho_{t+1} &\geq V^\rho_t, \quad \forall \rho \in \Delta(\mathcal{S}) \nonumber
}
\end{lemma}
\begin{proof}
 Follows from Lemma~\ref{Descent Property of PMD_interm} by an application of the Performance Difference Lemma~\ref{Performance difference lemma}, for an initial state distribution $\rho$
\eq{
V^\rho_{t+1}-V^\rho_t & =\frac{1}{1-\gamma} \mathbb{E}_{s \sim d^\rho_{t+1}}\left[\langle Q_t(s), \pi_{t+1}(s)-\pi_t(s)\rangle\right]\nonumber
\geq 0
}
\end{proof}

\subsubsection{Derivation of the Suboptimality Decomposition and Convergence of \PMD}\label{append:Suboptimality decomposition}
 \paragraph{Suboptimality decomposition}{
 Fix a state $s$. For any ${{\pi}(s)}, {\tilde{\pi}(s)}$, let $D_{-V}({{\pi}(s)}, {\tilde{\pi}(s)})$ be analogous to a standard Bregman divergence with mirror map $-V$, capturing the curvature of $-V$ at ${{\pi}(s)}$
 \eqq{
 D_{-V}({{\pi}(s)}, {\tilde{\pi}(s)})&\doteq
-V_{\pi}(s) - (- V_{\tilde{\pi}}(s)) - \langle -Q_{\tilde{\pi}}(s), {{\pi}(s)}-\tilde{\pi}(s)\rangle\nonumber
\\
&\doteq
-V_{\pi}(s) + V_{\tilde{\pi}}(s) + \langle Q_{\tilde{\pi}}(s)(s), {{\pi}(s)}-{\tilde{\pi}(s)}\rangle\nonumber
\\
 &=-\langle Q_{\pi}(s) -  Q_{\tilde{\pi}}(s), {{\pi}(s)}\rangle
\\
 &\text{(using Holder's inequality)}
\\
 &\geq -\|Q_{\pi}(s) -  Q_{\tilde{\pi}}(s)\|_{\infty} \|{{\pi}(s)}\|_1\nonumber
 \\
 &\geq
-\gamma \|V_{\pi} - V_{\tilde{\pi}}\|_{\infty} \label{eq:append:upper-approximation}
}
For the general case, using the approximation $\widehat{Q}_t(s)\approx Q_t(s)$, the per-iteration suboptimality is
\eqq{
V^{*}(s) -V_t(s) &= -\langle  \widehat{Q}_t(s),  {\pi}_{t+1}(s) - {\pi}^*(s)\rangle -\langle  \widehat{Q}_t(s),  {\pi}_t(s) - {\pi}_{t+1}(s)\rangle 
-  D_{-V}({\pi}^*(s), {\pi}_t(s))
\nonumber
\\
&\qquad- \langle Q_t(s)-\widehat{Q}_t(s),  {\pi}_t(s) - {\pi}^*(s)\rangle \label{eq:append:suboptimality}
}
The first  term, $-\langle  \widehat{Q}_t(s),  {\pi}_{t+1}(s) - {\pi}^*(s)\rangle$, is the forward regret (cf. \citet{JoulaniGS20}), defined as the regret of a ``cheating''  algorithm that uses the $\pi_{t+1}$ at time $t$, and depends only on the choices of the algorithm and the feedback it receives. This quantity can be can be upper-bounded using an idealized lookahead policy, $\bar{\pi}_{t+1}(s)$---greedy with respect to $\widehat{Q}_t(s)$ (cf. \citet{johnson2023optimal}). 

If $\pi_{t+1}(s)$ is the result of a PMD update, then \citet{johnson2023optimal} show that using the Three-Point Descent Lemma (Lemma 6,  \citet{xiao2022convergence}, included in Appendix~\ref{appendix:proofs_and_derivations:PMD}, Lemma~\ref{Three-Point Descent Lemma}), denoting the step sizes $\eta_{\pi_{t+1}} \geq 0$, the forward regret is further upper-bounded by 
\eqq{
-\langle  \widehat{Q}_t(s),  {\pi}_{t+1}(s) -{\pi}^*(s)\rangle &\leq 
-\langle \widehat{Q}_t(s), {\pi}_{t+1}(s)-\bar{\pi}_{t+1}(s) \rangle \nonumber
\\
&\leq 
\langle \widehat{Q}_t(s), \bar{\pi}_{t+1} - \pi_{t+1}\rangle \nonumber
\\
&\leq 
\eta^{\!-\!1}_{\pi_{t+1}} D_\Omega(\bar{\pi}_{t+1}(s), \pi_t(s)) - \eta^{\!-\!1}_{\pi_{t+1}} D_\Omega(\bar{\pi}_{t+1}(s), \pi_{t+1}(s))\nonumber
\\
&\qquad-\eta^{\!-\!1}_{\pi_{t+1}} D_\Omega({\pi}_{t+1}(s), \pi_t(s))\nonumber
\\
&\leq 
\eta^{\!-\!1}_{\pi_{t+1}} D_\Omega(\bar{\pi}_{t+1}(s), \pi_t(s)) \label{eq:append:bound on 1st term}
}
The second term in Eq.~\ref{eq:append:suboptimality} is
\eqq{
-\langle  \widehat{Q}_t(s),  \pi_t(s) - {\pi}_{t+1}(s)\rangle 
 = 
 V^{{t+1}}(s) -V_t(s)+  D_{-V}({\pi}_{t+1}(s), \pi_t(s))
 \nonumber
 \\
 \qquad + \langle Q_t(s)-\widehat{Q}_t(s),  \pi_t(s) - {\pi}_{t+1}(s)\rangle\label{eq:append:bound on 2nd term}
 }
 The third term $-D_{-V}({\pi}^*(s), {\pi}_t(s))$ can be bounded by applying the upper-approximation from Eq.\ref{eq:append:upper-approximation}, resulting in 
 \eqq{
 -D_{-V}({\pi}^*(s), {\pi}_t(s))\leq
\gamma \|V^{*} - V_t\|_{\infty} \label{eq:append:bound on 3rd term}
}
Plugging Eq.~\ref{eq:append:bound on 1st term}, ~\ref{eq:append:bound on 2nd term}, ~\ref{eq:append:bound on 3rd term} back in the suboptimality decomposition from Eq.~\ref{eq:append:suboptimality}, we obtain
\eqq{
V^{*}(s) -V_{t+1}(s) \leq \gamma \|V^{*} - V_t\|_{\infty} \nonumber
\\
\qquad+
\underbrace{\eta^{\!-\!1}_{\pi_{t+1}} D_\Omega(\bar{\pi}_{t+1}(s), \pi_t(s)) 
+\langle Q_t(s)-\widehat{Q}_t(s),  {\pi}^*(s) - {\pi}_{t+1}(s)\rangle + D_{-V}({\pi}_{t+1}(s), \pi_t(s))}_{\xi_{\pi_{t+1}} \text{(iteration error)}}
\label{eq:append:final_bound}
}
With $\xi_{\pi_{t+1}}$---the iteration error, recursing Eq.~\ref{eq:append:final_bound}
\eq{
\|V^{*}-V_t\|_{\infty}\leq \gamma^t \textstyle\sum_{i\leq t} \gamma^t \|V^{*}-V_{0}\|_{\infty} +  \nicefrac{\xi_i}{\gamma^i} 
}
}

\paragraph{Convergence of Exact PMD at $\gamma$-rate}\  
  If the PMD update is exact, then $\widehat{Q}_t(s)=Q_t(s), \forall s\in\S$.
  The Three-Point Descent Lemma~\ref{Three-Point Descent Lemma} guarantees policy improvement for an Exact PMD update, and yields Lemma~\ref{Descent Property of PMD} stating
  $V_{t+1}(s) \geq  V_t(s)$, and $\langle Q_{t+1}(s)-Q_t(s), {\pi}_{t+1}(s)\rangle\geq 0$. Consequently
  \eq{
  D_{-V}({\pi}_{t+1}(s), {\pi}_t(s))= -\langle Q_{t+1}(s)-Q_t(s),{{\pi}_{t+1}(s)}\rangle \leq 0
  }
  There remains only one term in the suboptimality from Eq.~\ref{eq:append:suboptimality}, namely
  \eq{
  \xi_{\pi_{t+1}} \leq \eta^{\!-\!1}_{\pi_{t+1}} D_\Omega(\bar{\pi}_{t+1}(s), \pi_t(s))
  }
  An optimal step-size $\eta_{\pi_{t+1}}$ can be derived by upper-bounding it $\eta^{\!-\!1}_{\pi_{t+1}} D_\Omega(\bar{\pi}_{t+1}(s), \pi_t(s))\leq\eps_{\pi_{t+1}}$, for any arbitrary constant $\eps_{\pi_{t+1}}$. 
  Setting $\eps_{\pi_{t+1}} = \gamma^{2(t+1)}\eps_0$ for some $\eps_0 > 0$, gives the optimal step-size with a geometrically increasing component, which guarantees  linear convergence at the $\gamma$-rate
  \eq{
\|V^*-V_t\|_{\infty} \leq \gamma^t\left(\left\|V^*- V_{0}\right\|_{\infty}+\nicefrac{\eps_0}{1-\gamma}\right)
}
matching the bounds of PI and VI as $\eps_0$ goes to $0$ (cf. Theorem 4.1., \citet{johnson2023optimal}).

\subsection{Proofs for Sec.~\ref{sec:Functional Acceleration for PMD}: Functional Acceleration for PMD}\label{append:Proofs for Sec. Functional Acceleration for PMD}
    
\begin{definition}{\textbf{(Functional gradient of the Bregman divergence )}}\label{Functional gradient of the Bregman divergence} Fix a state $s$. For any policies $\pi_1, \pi_0$, we denote the gradient of the Bregman divergence with respect to the first argument
    \eq{
     \nabla D_\Omega({\pi_1(s)}, {\pi_0(s)})
    \doteq \nabla \Omega(\pi_1(s)) - \nabla \Omega (\pi_0(s))
    }
\end{definition}

The following lemma can be also interpreted as a definition for the difference of differences of Bregman divergences.
\begin{lemma}{\textbf{(Four-Point Identity Lemma of Bregman divergences)}}\label{Four-Point Identity Lemma of Bregman divergences}
For any four policies $\pi_3, \pi_2, \pi_1, \pi_0$, we have
    \eq{
 \langle \nabla D_\Omega({\pi_1(s)}, {\pi_0(s)}),\pi_3(s) - \pi_2(s)\rangle =   D_{\Omega}(\pi_3(s), {\pi_0(s)})-D_{\Omega}(\pi_3(s),{\pi_1(s)}) 
 \\
- [D_{\Omega}(\pi_2(s),{\pi_0(s)}) - D_{\Omega}(\pi_2(s),{\pi_1(s)}) ]
    }
\end{lemma}
\begin{proof} 
Immediate from the definition.
\end{proof} 
An immediate consequence is the Three-Point Identity Lemma of Bregman divergences (cf. \citet{Bubeck15}, Eq. 4.1, \citet{BeckT03}, Lemma 4.1).

\begin{lemma}{\textbf{(Three-Point Identity Lemma of Bregman divergences)}}\label{Three-Point Identity Lemma of Bregman divergences} For any
three policies $\pi_2, \pi_1, \pi_0$, 
\eq{
\langle\nabla D_\Omega({\pi_1(s)},{\pi_0(s)}), \pi_2(s)-{\pi}_{1}(s)\rangle&=D_{\Omega}(\pi_2(s), {{\pi}_{0}(s)})-D_{\Omega}(\pi_2(s), {{\pi}_{1}(s)})
-D_{\Omega}({\pi}_{1}(s), {\pi}_{0}(s))
}
\end{lemma}
\begin{proof}
Apply Lemma~\ref{Four-Point Identity Lemma of Bregman divergences} with ${\pi}_3={\pi}_{2}$, ${\pi}_2={\pi}_{1}$.
\end{proof}

\begin{corollary}{\textbf{(Three-Point Descent Corollary for PMD)}}\label{Three-Point Descent Corollary for PMD} Consider the policies produced by the iterative updates of exact \PMD in Eq.\ref{eq:exact PMD}. For any policy $\pi$,  timestep $t\geq 0$, and state $s$
\eq{
 \langle \eta_{\pi_{t+1}} Q_t(s), \pi(s)-\pi_{t+1}(s) \rangle \leq \langle \nabla h(\pi_{t+1}(s))-\nabla h(\pi_{t}(s)), \pi(s)-\pi_{t+1}(s) \rangle
}
\end{corollary}
\begin{proof}
From the Three-Point descent Lemma ~\ref{Three-Point Descent Lemma} (Eq.\ref{eq:exact PMD})
\eq{
-\eta_{\pi_{t+1}} \langle  Q_t(s), \pi_{t+1}(s) 
 - \pi(s)\rangle\leq
D_\Omega(\pi(s), {\pi}_{t}(s))
 -D_\Omega(\pi(s), {\pi}_{t+1}(s))  
 - D_\Omega(\pi_{t+1}(s), \pi_{t}(s))
}
apply Lemma~\ref{Three-Point Identity Lemma of Bregman divergences} on the right hand side, with $\pi_0=\pi_t,\pi_1=\pi_{t+1}, \pi_2=\pi$, which yields the claimed inequality.
\end{proof}

\begin{lemma}{{\textbf{(Extrapolation from the future: Equivalence of updates)}}}\label{appendix:Forward extrapolation}
For any state $s$ and timestep $t> 0$, the update of \pmdext, $\pi_{t+1}=\argmin_{\pi(s)\in\Delta(\A)} -\langle (\mathcal{T}_{\mu_t}{Q}_t)(s) \pi(s)\rangle  +  \eta_{\pi_{t+1}}^{-1}D_h(\pi(s), \pi_t(s))$,  can be rewritten as a correction to the update $\mu_t=\argmin_{\pi(s)\in\Delta(\A)} - \langle Q_{t}(s), \pi(s)\rangle + \eta^{\!-\!1}_{\mu_t} D_h(\pi(s), \pi_t(s))$
\eq{
\pi_{t+1}(s) &\doteq \argmin_{\pi(s)\in\Delta(\A)} -\langle (\mathcal{T}_{\mu_t}{Q}_t)(s) \!-\!\eta_{\pi_{t+1}}^{-1}{{\eta}_{\mu_t}}{{Q}_t(s)},\pi(s)\rangle + \eta_{\pi_{t+1}}^{-1}D_h(\pi(s),\mu_t(s))
}
\end{lemma}
\begin{proof}
For any timestep $t> 0$, and state $s$, from the definition of \pmdext
\eq{
\mu_{t}(s) &=  \argmin_{\pi(s)\in\Delta(\A)} - \langle Q_{t}(s), \pi(s)\rangle + \eta^{\!-\!1}_{\mu_t} D_h(\pi(s), \pi_t(s))
}
Applying Corollary~\ref{Three-Point Descent Corollary for PMD} for $\pi_{t+1} = \mu_t$, and for any policy $\pi$, we have
\eqq{
 \langle \eta_{\mu_{t}} Q_t(s) -\nabla h(\mu_t)(s)), \pi(s)-\mu_t(s) \rangle \leq -\langle \nabla h(\pi_{t}(s)), \pi(s)-\mu_t(s) \rangle\label{eq:to plug in}
}
Plugging Eq.\ref{eq:to plug in} in the proposed update, we recover the formulation of \pmdext
\eq{
\pi_{t+1}(s) &\doteq \argmin_{\pi(s)\in\Delta(\A)} -\langle (\mathcal{T}_{\mu_t}{Q}_t)(s) \!-\!\eta_{\pi_{t+1}}^{-1}{{\eta}_{\mu_t}}{{Q}_t(s)},\pi(s)\rangle + \eta_{\pi_{t+1}}^{-1}D_h(\pi(s),\mu_t(s))
\\
&=\argmin_{\pi(s)\in\Delta(\A)} -\langle (\mathcal{T}_{\mu_t}{Q}_t)(s) \!-\!\eta_{\pi_{t+1}}^{-1}{{\eta}_{\mu_t}}{{Q}_t(s)},\pi(s)\rangle
\\
&\qquad\qquad\qquad\qquad + \eta_{\pi_{t+1}}^{-1}( h(\pi(s)) - h(\mu_t(s)) - \langle\nabla h(\mu_t(s)), \pi(s) - \mu_t(s)\rangle)
\\
&=\argmin_{\pi(s)\in\Delta(\A)} -\langle (\mathcal{T}_{\mu_t}{Q}_t)(s), \pi(s)\rangle  + \eta_{\pi_{t+1}}^{-1}( h(\pi(s)) - h(\mu_t(s)))
\\
&\qquad\qquad\qquad\qquad + \eta_{\pi_{t+1}}^{-1}{{\eta}_{\mu_t}}\langle {{Q}_t(s)},\pi(s)\rangle  - \langle\nabla h(\mu_t(s)), \pi(s) - \mu_t(s)\rangle)
\\
&\text{(adding the constant term $-\eta_{\pi_{t+1}}^{-1}{{\eta}_{\mu_t}}\langle {{Q}_t(s)}, \mu_t(s)\rangle$)}
\\
&=\argmin_{\pi(s)\in\Delta(\A)} -\langle (\mathcal{T}_{\mu_t}{Q}_t)(s), \pi(s)\rangle  \eta_{\pi_{t+1}}^{-1}( h(\pi(s)) - h(\mu_t(s)))
\\
&\qquad\qquad\qquad\qquad + \eta_{\pi_{t+1}}^{-1}{{\eta}_{\mu_t}}\langle {{Q}_t(s)}-\nabla h(\mu_t(s)), \pi(s) - \mu_t(s)\rangle
\\
&\text{(plugging in Eq.~\ref{eq:to plug in})}
\\
&=\argmin_{\pi(s)\in\Delta(\A)} -\langle (\mathcal{T}_{\mu_t}{Q}_t)(s), \pi(s)\rangle  +  \eta_{\pi_{t+1}}^{-1}( h(\pi(s)) - h(\mu_t(s)) - \langle\nabla h(\pi_t(s)), \pi(s)- \mu_t(s) \rangle)
\\
&\text{(swapping constant terms $h(\mu_t(s))$ with $h(\pi_t(s))$ and $\langle\nabla h(\pi_t(s)),\mu_t(s) \rangle$ with $\langle\nabla h(\pi_t(s)),\pi_t(s) \rangle$)}
\\
&=\argmin_{\pi(s)\in\Delta(\A)} -\langle (\mathcal{T}_{\mu_t}{Q}_t)(s) \pi(s)\rangle  +  \eta_{\pi_{t+1}}^{-1}( h(\pi(s)) - h(\pi_t(s)) - \langle\nabla h(\pi_t(s)), \pi(s)- \pi_t(s) \rangle)
\\
&=\argmin_{\pi(s)\in\Delta(\A)} -\langle (\mathcal{T}_{\mu_t}{Q}_t)(s) \pi(s)\rangle  +  \eta_{\pi_{t+1}}^{-1}D_h(\pi(s), \pi_t(s))
}
\end{proof}

\subsubsection{Extrapolation from the past: derivation of the lazy momentum-based PMD algorithm}\label{Extrapolation from the past: derivation of the lazy momentum-based PMD algorithm}
 
We take the following steps to arrive at the proposed algorithm. 

First, we rely on the equivalence in Lemma~\ref{appendix:Forward extrapolation} to rewrite the update to $\pi_{t+1}$ in \pmdext\ 
\eqq{
\textstyle \pi_{t+1}(s) &=  \argmin_{\pi(s)\in\Delta(\A)} \!-\! \langle (\mathcal{T}_{\mu_t}{Q}_{t})(s)  \!-\! \eta^{\!-\!1}_{\pi_{t+1}} {\eta}_{\mu_t} {Q}_t(s), \pi(s)\rangle +\eta^{\!-\!1}_{\pi_{t+1}}D_h(\pi(s), \mu_t(s)) \label{eq:bw_cor}
}
Next, to remedy the need for two evaluation procedures, we recycle evaluation updates from previous steps. Specifically, we replace $Q_{t}$ with the closest past evaluation, i.e. $\mathcal{T}_{\mu_{t-1}} Q_{t-1}$, twice, in Eq.~\ref{eq:bw_cor} (the equivalent update to $\pi_{t+1}$ from \pmdext) and in the update to $\mu_{t}$ of \pmdext.
This leads to
\eqq{
 \mu_{t}(s) &\doteq  \argmin_{\pi(s)\in\Delta(\A)}  \!-\! \langle (\mathcal{T}_{\mu_{t-1}} Q_{t-1})(s), \pi(s)\rangle + \eta^{\!-\!1}_{\pi_{t+1}} D_h(\pi(s), \pi_t(s))
 \label{eq:lazy fw_cor} 
 \\
\textstyle \pi_{t+1}(s) &=  \argmin_{\pi(s)\in\Delta(\A)} \!-\! \langle (\mathcal{T}_{\mu_t}{Q}_{t})(s)  \!-\! \eta^{\!-\!1}_{\pi_{t+1}} {\eta}_{\mu_t} (\mathcal{T}_{\mu_{t-1}} Q_{t-1})(s), \pi(s)\rangle +{\eta}^{\!-\!1}_{\mu_t
}D_h(\pi(s), \mu_t(s)) 
\label{eq:lazy bw_cor}
}
We now apply the equivalence shown in Lemma~\ref{appendix:Backward extrapolation} (which is analogous to Lemma~\ref{appendix:Forward extrapolation}, but time-reversed), yielding an update written in terms of one policy variable $\mu$ and $\mathcal{T}_{\mu_t}{Q}_{t}$
\eq{
\mu_{t+1}(s) \!=\! 
\argmin_{\pi(s)\in\Delta(\A)}
\!-\! \langle (\mathcal{T}_{\mu_t}{Q}_{t})(s)  \!+\! {\eta^{\!-\!1}_{\mu_{t+1}}}\eta_{\pi_{t+1}}((\mathcal{T}_{\mu_t}{Q}_{t})(s)  \!-\! (\mathcal{T}_{\mu_{t-1}}{Q}_{t-1})(s)), \pi(s)\rangle 
\\
\quad \!+{\eta^{\!-\!1}_{\mu_{t+1}}}D_h(\pi(s), \mu_t(s))  
}
For consistency with \pmdloo\ and \pmdext, we henceforth replace these variables ($\mu$ and $\mathcal{T}_{\mu} Q$) with $\pi$ and $Q$. In this case, $Q =Q_\pi$, and instead of $\eta^{\!-\!1}_{\mu_{t+1}}\eta_{\pi_{t+1}}$, we then use two consecutive step-sizes $\eta^{\!-\!1}_{\pi_{t+1}}\eta_{\pi_{t}}$, yielding 
\eqq{
\pi_{t+1}(s) &\!=\! 
\argmin_{\pi(s)\in\Delta(\A)}
\!-\! \langle {Q}_{t}(s)  \!+\! \eta^{\!-\!1}_{\pi_{t+1}}\eta_{\pi_{t}}({Q}_{t}(s)  \!-\! {Q}_{t-1}(s)), \pi(s)\rangle \!+ \eta^{\!-\!1}_{\pi_{t+1}} D_h(\pi(s), \pi_t(s)) 
}

\begin{lemma}{\textbf{{(Extrapolation from the past: equivalence of updates)}}}
\label{appendix:Backward extrapolation}
For any state $s$ and timestep $t> 0$, the updates 
\eqq{
 \mu_{t}(s) &\doteq  \argmin_{\pi(s)\in\Delta(\A)}  \!-\! \langle (\mathcal{T}_{\mu_{t-1}} Q_{t-1})(s), \pi(s)\rangle + {\eta^{\!-\!1}_{\mu_t}}D_h(\pi(s), \pi_t(s))
 \label{eq:lazy fw_cor2} 
 \\
\textstyle \pi_{t+1}(s) &=  \argmin_{\pi(s)\in\Delta(\A)} \!-\! \langle (\mathcal{T}_{\mu_t}{Q}_{t})(s)  \!-\! \eta^{\!-\!1}_{\pi_{t+1}} {\eta}_{\mu_t} (\mathcal{T}_{\mu_{t-1}} Q_{t-1})(s), \pi(s)\rangle +\eta^{\!-\!1}_{\pi_{t+1}}D_h(\pi(s), \mu_t(s)) 
\label{eq:lazy bw_cor2}
}
can be rewritten as:
\eq{
\mu_{t+1}(s) &=  \argmin_{\pi(s)\in\Delta(\A)}  \!-\! \langle (\mathcal{T}_{\mu_t}{Q}_{t})(s) + {\eta^{\!-\!1}_{\mu_{t+1}}}\eta_{\pi_{t+1}}\delta_{t}(s), \pi(s)\rangle + {\eta^{\!-\!1}_{\mu_{t+1}}}D_h(\pi(s), h(\mu_{t}(s))) 
}
\end{lemma}
\begin{proof}
First, we increment the index $t$ in equation Eq.~\ref{eq:lazy fw_cor2}
\eqq{
\mu_{t+1}(s) &\doteq  \argmin_{\pi(s)\in\Delta(\A)}  \!-\! \langle (\mathcal{T}_{\mu_t}{Q}_{t})(s), \pi(s)\rangle + {\eta^{\!-\!1}_{\mu_{t+1}}}D_h(\pi(s), \pi_{t+1}(s))
 \label{eq:lazy fw_cor3} 
}
Then, we apply Corollary~\ref{Three-Point Descent Corollary for PMD} to $\pi_{t+1}$, and for any policy $\pi$, with the shorthand notation $\delta_{t}(s) \doteq \mathcal{T}_{\mu_t}{Q}_{t})(s)  \!-\!\eta^{\!-\!1}_{\pi_{t+1}} {\eta}_{\mu_t}  (\mathcal{T}_{\mu_{t-1}}{Q}_{t-1})(s)$, this yields
\eqq{
 \langle \eta_{\pi_{t+1}} \delta_{t}(s) -\nabla h(\pi_{t+1})(s)), \pi(s)-\pi_{t+1}(s) \rangle \leq -\langle \nabla h(\mu_{t}(s)), \pi(s)-\pi_{t+1}(s) \rangle\label{eq:to plug in2}
}
Plugging Eq.~\ref{eq:to plug in2} in back in Eq.~\ref{eq:lazy fw_cor3}
\eq{
\mu_{t+1}(s) &=  \argmin_{\pi(s)\in\Delta(\A)}  \!-\! \langle (\mathcal{T}_{\mu_t}{Q}_{t})(s), \pi(s)\rangle 
\\
&\qquad\qquad\qquad\qquad + {\eta^{\!-\!1}_{\mu_{t+1}}}(h(\pi(s)) - h(\pi_{t+1}(s)) - \langle \nabla h(\pi_{t+1}(s)),\pi(s) -\pi_{t+1}(s)\rangle
\\
 &=  \argmin_{\pi(s)\in\Delta(\A)}  \!-\! \langle (\mathcal{T}_{\mu_t}{Q}_{t})(s) + \eta^{\!-\!1}_{\mu_{t+1}}\eta_{\pi_{t+1}}\delta_{t}(s), \pi(s)\rangle + \eta^{\!-\!1}_{\mu_{t+1}}(h(\pi(s)) - h(\pi_{t+1}(s))) 
 \\
 &\qquad\qquad\qquad\qquad +{\eta^{\!-\!1}_{\mu_{t+1}}}\eta_{\pi_{t+1}}\langle \delta_{t}(s), \pi(s)\rangle - {\eta^{\!-\!1}_{\mu_{t+1}}}\langle \nabla h(\pi_{t+1}(s)),\pi(s) -\pi_{t+1}(s)\rangle
 \\
&\text{(adding the constant term $-\eta^{\!-\!1}_{\mu_{t+1}}\eta_{\pi_{t+1}}\langle \delta_{t}(s), \pi_{t+1}(s)\rangle$)}
 \\
 &=  \argmin_{\pi(s)\in\Delta(\A)}  \!-\! \langle (\mathcal{T}_{\mu_t}{Q}_{t})(s) + {\eta^{\!-\!1}_{\mu_{t+1}}}\eta_{\pi_{t+1}}\delta_{t}(s), \pi(s)\rangle + {\eta^{\!-\!1}_{\mu_{t+1}}}(h(\pi(s)) - h(\pi_{t+1}(s))) 
 \\
 &\qquad\qquad\qquad\qquad +{\eta^{\!-\!1}_{\mu_{t+1}}}\eta_{\pi_{t+1}}\langle \delta_{t}(s) - {\eta^{\!-\!1}_{\mu_{t+1}}} \nabla h(\pi_{t+1}(s)),\pi(s) -\pi_{t+1}(s)\rangle
  \\
&\text{(applying Eq.\ref{eq:to plug in2})}
 \\
 &=  \argmin_{\pi(s)\in\Delta(\A)}  \!-\! \langle (\mathcal{T}_{\mu_t}{Q}_{t})(s) + {\eta^{\!-\!1}_{\mu_{t+1}}}\eta_{\pi_{t+1}}\delta_{t}(s), \pi(s)\rangle + {\eta^{\!-\!1}_{\mu_{t+1}}}(h(\pi(s)) - h(\pi_{t+1}(s))) 
 \\
 &\qquad\qquad\qquad\qquad - {\eta^{\!-\!1}_{\mu_{t+1}}}\langle \nabla h(\mu_{t}(s)),\pi(s) -\pi_{t+1}(s)\rangle
   \\
&\text{(swapping constant terms)}
 \\
 &=  \argmin_{\pi(s)\in\Delta(\A)}  \!-\! \langle (\mathcal{T}_{\mu_t}{Q}_{t})(s) + {\eta^{\!-\!1}_{\mu_{t+1}}}\eta_{\pi_{t+1}}\delta_{t}(s), \pi(s)\rangle + {\eta^{\!-\!1}_{\mu_{t+1}}}(h(\pi(s)) - h(\mu_{t}(s))) 
 \\
 &\qquad\qquad\qquad\qquad - {\eta^{\!-\!1}_{\mu_{t+1}}}\langle \nabla h(\mu_{t}(s)),\pi(s) -\mu_{t}(s)\rangle
 \\
 &=  \argmin_{\pi(s)\in\Delta(\A)}  \!-\! \langle (\mathcal{T}_{\mu_t}{Q}_{t})(s) + {\eta^{\!-\!1}_{\mu_{t+1}}}\eta_{\pi_{t+1}}\delta_{t}(s), \pi(s)\rangle + {\eta^{\!-\!1}_{\mu_{t+1}}}D_h(\pi(s), h(\mu_{t}(s)))
}

\end{proof}

\begin{lemma}\label{lemma:descent property of PMD+}\textbf{(Descent Property of exact \pmdloo)} Consider the policies produced by the iterative updates of \pmdloo
\eq{
\pi_{t+1}(s) = \argmin_{\pi(s) \in \Delta(\mathcal{A})} \langle \mathcal{T}_{\mu_t}{Q}_t(s), \pi(s)\rangle
+{\eta^{\!-\!1} _{\pi_{t+1}}}D_h(\pi(s), \pi_t(s))
}
where $(\mathcal{T}_{\mu_t}{Q}_t)(s) \doteq \E[r(s) + \gamma \langle {Q}_{t}({s^\prime}), \mu_t({s^\prime})\rangle]$ is the lookahead, $\mu_t(s)$ is an intermediary policy 
greedy with respect to $Q_{t}({s})$,  and ${\eta_{\pi_{t+1}}}\geq 0$ is a step-size. Then, for any timestep $t\geq 0$
\eqq{
\langle (\mathcal{T}_{\mu_t}{Q}_t)(s), \pi_{t+1}(s) - \mu_t(s)\rangle \geq 0, \forall s\in \S \label{eq:descent_propery1}
\\
\langle Q_{t+1}(s) - (\mathcal{T}_{\mu_t}{Q}_t)(s),\pi_{t+1}(s) \rangle \geq 0, \forall s\in \S \label{eq:descent_propery2}
}
\end{lemma}
\begin{proof}
Consider first the descent property of $\mu_t$
\eqq{
\langle {Q}_{t}({s})-(\mathcal{T}_{\mu_t}{Q}_t)({s}),\mu_t({s})\rangle &= \sum_{a\in\A} (Q_{{t}}({s,a}) -(\mathcal{T}_{\mu_t}{Q}_t)({s,a}))\mu_t({a|s)}\nonumber
\\
&= \gamma  \sum_{s^\prime\in\S}  \sum_{a\in\A}  P({s^{\prime}|s, a}) \mu_t({a|s}) \left[\langle  Q_{t}({s^\prime}),\pi_{t}({s^\prime})\rangle- \langle {Q}_{t}({s^\prime}),\mu_t({s^\prime})\rangle\right]
 \leq 0 \label{eq:interm}
}
where the last inequality follows from the definition of $\mu_t(s)$ as greedy with respect to $Q_{t}(s)$, which implies $\langle {Q}_{t}({s^\prime}),\mu_t({s^\prime})\rangle \geq \langle {Q}_{t}({s^\prime}),\pi_{t}({s^\prime})\rangle$.

Then, for the descent property of $\pi_{t+1}(s)$, we have
\eqq{
\langle &Q_{t+1}(s) - (\mathcal{T}_{\mu_t} Q_t)(s), \pi_{t+1}(s)\rangle 
\\
&= \sum_{a\in\A} (Q_{t+1}({s,a}) -(\mathcal{T}_{\mu_t} Q_t)({s,a}))\pi_{t+1}({a|s})\nonumber
\\
 &= \gamma  \sum_{s^\prime\in\S}  \sum_{a\in\A}  P({s^{\prime}|s, a}) \pi_{t+1}({a|s}) \left[\langle  Q_{t+1}(s^\prime),\pi_{t+1}({s^\prime})\rangle- \langle {Q}_{t}({s^\prime}),\mu_t({s^\prime})\rangle\right]\nonumber
 \\
 &= \gamma  \sum_{s^\prime\in\S}  \sum_{a\in\A}  P({s^{\prime}|s, a}) \pi_{t+1}({a|s}) \left[\langle  Q_{t+1}(s^\prime),\pi_{t+1}({s^\prime})\rangle- \langle (\mathcal{T}_{\mu_t} Q_t)({s^\prime}),\mu_t({s^\prime})\rangle 
 - \langle {Q}_{t}({s^\prime})-(\mathcal{T}_{\mu_t} Q_t)({s^\prime}),\mu_t({s^\prime})\rangle\right]\nonumber
\\
&\overset{(Eq.~\ref{eq:interm})}{\geq} 
\gamma  \sum_{s^\prime\in\S}  \sum_{a\in\A}  P({s^{\prime}|s, a}) \pi_{t+1}({a|s}) \left[\langle  Q_{t+1}(s^\prime),\pi_{t+1}({s^\prime})\rangle- \langle (\mathcal{T}_{\mu_t} Q_t)({s^\prime}),\mu_t({s^\prime})\rangle\right]\nonumber
\\
  &= \gamma  \sum_{s^{\prime}} \sum_a  P({s^{\prime}|s, a}) \pi_{t+1}({a|s}) \left[\langle (\mathcal{T}_{\mu_t} Q_t)({s^\prime}), \pi_{t+1}({s^\prime})-\mu_t({s^\prime})\rangle+\langle Q_{t+1}(s^\prime)-(\mathcal{T}_{\mu_t} Q_t)({s^\prime}),\pi_{t+1}({s^\prime})\rangle\right]\nonumber
}
Recursing, yields
\eqq{
\langle Q_{t+1}(s) - (\mathcal{T}_{\mu_t} Q_t)(s), \pi_{t+1}(s)\rangle 
&=
\nicefrac{\gamma}{1-\gamma} \sum_{s^{\prime}} d_{t+1} ({s^\prime\to s^{\prime}})\langle (\mathcal{T}_{\mu_t} Q_t)({s^{\prime}}), \pi_{t+1}({s^{\prime}})-\mu_t({s^{\prime}})\rangle   \label{last:pdl}
}
From Lemma~\ref{Three-Point Descent Lemma} with 
$\pi=\mu_t$ and $\eta_{\pi_{t+1}}\geq 0$
\eqq{
\langle (\mathcal{T}_{\mu_t} Q_t)_{s}, \pi_{t+1}(s)-\mu_t({s})\rangle \geq \eta^{\!-\!1}_{\pi_{t+1}} \left(D_\Omega(\mu_t(s), \pi_{t+1}(s))+D_\Omega(\pi_{t+1}(s), \mu_t(s))\right) \geq 0 \label{eq:geq0}
}
This proves the first claim in Eq.~\ref{eq:descent_propery1}.
Plugging Eq.~\ref{eq:geq0} back in Eq.~\ref{last:pdl} yields the second claim in Eq.~\ref{eq:descent_propery2}, $\langle Q_{t+1}(s) - (\mathcal{T}_{\mu_t} Q_t)(s), \pi_{t+1}(s)\rangle 
\geq0$
\end{proof}

\begin{theorem} {\textbf{(Functional acceleration with \pmdloo)}}
\label{append:Functional acceleration with PMDloo}
Consider the policies produced by the iterative updates of \pmdloo
\eqq{
\pi_{t+1}(s) &=  \argmin_{\pi(s)\in\Delta(\A)} - \langle (\mathcal{T}_{\mu_t} Q_t)(s), \pi(s)\rangle + \eta^{\!-\!1}_{\pi_{t+1}} D_h(\pi(s), \pi_t(s))
}
where  $(\mathcal{T}_{\mu_t} Q_t)(s) \doteq \E[r(s) + \gamma \langle{Q}_t({s^\prime}), \mu_t({s^\prime})\rangle]$, $\mu_t(s)$ is greedy with respect to $Q_{t}({s})$,  $\eta_{\pi_{t+1}}\geq 0$ are adaptive step sizes, such that $\forall \eps_{\pi_{t+1}}$ arbitrarily small, ${\eta^{\!-\!1}_{\pi_{t+1}}}D_{h}(\operatorname{greedy}((\mathcal{T}_{\mu_t} Q_t)(s)),\pi_t(s))\leq {\eps_{\pi_{t+1}}}$.
Then,
\eqq{
 V^{*}(s) -V_{t+1}(s)&\leq  \gamma^2 \| V^{*}  - V_t\|_{\infty} + \eps_{\pi_{t+1}} \label{eq:mini_claim}
}
and recursing yields
\eqq{
\textstyle\|V^{*} -V_t\|_{\infty} &\leq (\gamma^{2})^t (\| V^{*}  - V_0\|_{\infty} + \sum_{i\leq t} \nicefrac{\eps_{\pi_{i+1}}}{(\gamma^2)^i} )\label{eq:recursive_claim}
}
\end{theorem}
\begin{proof}
If $\pi_{t+1}$ is the result of a PMD update which uses $(\mathcal{T}_{\mu_t} Q_t)(s)$, and step-sizes  $\eta_{\pi_{t+1}}$, then applying Lemma~\ref{Three-Point Descent Lemma} for $\pi(s) = \bar{\pi}_{t+1}(s)$ greedy with respect to $(\mathcal{T}_{\mu_t} Q_t)(s)$
\eqq{
-\langle (\mathcal{T}_{\mu_t} Q_t)(s) , \pi_{t+1}(s) - \bar{\pi}_{t+1}(s)\rangle
&\leq \eta^{\!-\!1}_{\pi_{t+1}}\big(D_{h}(\bar{\pi}_{t+1}(s),\pi_t(s)) - D_{h}(\pi_{t+1}(s),\pi_t(s)) \nonumber
\\
&\qquad - D_{h}(\pi_{t+1}(s),\bar{\pi}_{t+1}(s)) \big)\nonumber
\\
&\leq \eta^{\!-\!1}_{\pi_{t+1}}D_{h}(\bar{\pi}_{t+1}(s),\pi_t(s)) \label{eq:final_ineq}
}
Further, the suboptimality is
\eqq{
V^{*}(s) -V_{t+1}(s) &= -\langle  (\mathcal{T}_{\mu_t} Q_t)(s), {\pi}_{t+1}(s) - {\pi}^*(s)\rangle -  \langle Q_{t+1}(s) - (\mathcal{T}_{\mu_t} Q_t)(s), \pi_{t+1}(s)\rangle\nonumber
\\
&\qquad
+\langle Q^{*}(s) - (\mathcal{T}_{\mu_t} Q_t)(s), \pi^*(s)\rangle \label{eq:final_eq}
}
Since $\langle  (\mathcal{T}_{\mu_t} Q_t)(s),{{\pi}^*(s)}\rangle \leq \langle  (\mathcal{T}_{\mu_t} Q_t)(s),{\bar{\pi}_{t+1}(s)}\rangle$ if ${\bar{\pi}_{t+1}(s)}$ is greedy with respect to $(\mathcal{T}_{\mu_t} Q_t)(s)$, then plugging Eq~\ref{eq:final_ineq} in Eq~\ref{eq:final_eq}
\eqq{
V^{*}(s) -V_{t+1}(s)&\leq  \nicefrac{1}{{\eta}^t}D_{h}(\bar{\pi}_{t+1}(s),\pi_t(s))  -  \langle Q_{t+1}(s) - (\mathcal{T}_{\mu_t} Q_t)(s), \pi_{t+1}(s)\rangle\nonumber
\\
&\qquad
+\langle Q^{*}(s) - (\mathcal{T}_{\mu_t} Q_t)(s), \pi^*(s)\rangle \label{eq:final_last_eq}
}
Next, cf. Lemma~\ref{lemma:descent property of PMD+}, 
$\langle Q_{t+1}(s) - (\mathcal{T}_{\mu_t} Q_t)(s), \pi_{t+1}(s) \rangle\geq 0$.
Plugging this back into Eq~\ref{eq:final_eq}
\eqq{
V^{*}(s) -V_{t+1}(s)&\leq  \eta^{\!-\!1}_{\pi_{t+1}}D_{h}(\bar{\pi}_{t+1}(s),\pi_t(s)) 
+\langle Q^{*}(s) - (\mathcal{T}_{\mu_t} Q_t)(s), \pi^*(s)\rangle \nonumber
\\
&\leq  \langle Q^{*}(s) - (\mathcal{T}_{\mu_t} Q_t)(s), \pi^*(s)\rangle + \eps_{\pi_{t+1}} \label{eq:to_recurs}
}
where the last step follows from step-size adaptation condition ${\eta^{\!-\!1}_{\pi_{t+1}}}D_{h}(\bar{\pi}_{t+1}(s),\pi_t(s))\leq {\eps_{\pi_{t+1}}}$.

Decomposing the remaining term
\eq{
\langle Q^{*}(s) - (\mathcal{T}_{\mu_t} Q_t)(s), {\pi}^{*}(s)\rangle
&= \sum_{a\in\A} (Q^*({s,a}) -(\mathcal{T}_{\mu_t} Q_t)({s,a}))\pi^{*}({a|s})\nonumber
\\
 &= \gamma  \sum_{s^\prime\in\S}  \sum_{a\in\A}  P({s^{\prime}|s, a}) \pi^{*}({a|s}) \left[\langle  Q^{*}({s^\prime}),\pi^{*}({s^\prime})\rangle- \langle {Q}_{t}({s^\prime}),\mu_t({s^\prime})\rangle\right]\nonumber
 \\
  &= \gamma  \sum_{s^\prime\in\S}  \sum_{a\in\A}  P({s^{\prime}|s, a}) \pi^{*}({a|s}) \left[\langle  Q^{*}({s^\prime})-  {Q}_{t}({s^\prime}),{\pi}^*({s^\prime})\rangle- \langle {Q}_{t}({s^\prime}),\mu_t({s^\prime})-{\pi}^*({s^\prime})\rangle\right]\nonumber
 \\
  &\leq \gamma  \sum_{s^\prime\in\S}  \sum_{a\in\A}  P({s^{\prime}|s, a}) \pi^{*}({a|s}) \left[\langle  Q^{*}({s^\prime})-  {Q}_{t}({s^\prime}),{\pi}^*({s^\prime})\rangle\right]
  \\
  &= \gamma^2  \sum_{s^\prime\in\S}  \sum_{a\in\A} \pi^{*}({a|s}) P({s^{\prime}|s, a}) \sum_{s^{\prime\prime}\in\S} \sum_{a^\prime\in\A} {\pi}^*({s^\prime,a^\prime})P({s^{\prime\prime} | s^\prime, a^\prime})\left[V^{*}({s^{\prime\prime}})-  {V}_{t}({s^{\prime\prime}})\right]
}
where the inequality follows due to $\mu_t(s)$ being greedy with respect to $Q_{t}(s)$, $\forall s\in\S$ by definition.

Taking the max norm and applying the triangle inequality and the contraction property
\eq{
\|\langle Q^{*} - (\mathcal{T}_{\mu_t} Q_t), \pi^*\rangle \|_{\infty}
&\leq  \gamma^2 \| V^{*}  - V_t\|_{\infty} 
}
and then plugging this back in Eq.~\ref{eq:to_recurs} 
\eq{
V^{*}(s) -V_{t+1}(s)
&\leq  \|\langle Q^{*}(s) - (\mathcal{T}_{\mu_t} Q_t), \pi^*\rangle \|_{\infty} + \eps_{\pi_{t+1}} 
\\
&\leq  \gamma^2 \| V^{*}  - V_t\|_{\infty}  + \eps_{\pi_{t+1}} 
}
which is the first claim in Eq.\ref{eq:mini_claim}. Then recursing yields the second claim in Eq.\ref{eq:recursive_claim}.
\end{proof}

\begin{theorem}{\textbf{(Functional acceleration with \pmdext)}}
\label{append:Functional acceleration with PMDext}
Consider the policies produced by the iterative updates of \pmdext
\eqq{
\mu_t({s}) &=  \argmin_{\pi(s)\in\Delta(\A)} - \langle {Q}_{t}(s), \pi(s)\rangle + {\eta^{\!-\!1}_{\mu_t}}D_h(\pi(s), \pi_t(s))
\\
\pi_{t+1}(s) &=  \argmin_{\pi(s)\in\Delta(\A)} - \langle (\mathcal{T}_{\mu_t} Q_t)(s), \pi(s)\rangle + \eta^{\!-\!1}_{\pi_{t+1}} D_h(\pi(s), \pi_t(s)) 
}
where $(\mathcal{T}_{\mu_t} Q_t)(s) \doteq \E[r(s) + \gamma \langle{Q}_t({s^\prime}), \mu_t({s^\prime})\rangle]$ is the lookahead, $\eta_{\pi_{t+1}},\eta_{\mu_t}\geq 0$ are adaptive step sizes, such that $\forall \eps_{\pi_{t+1}}, \eps_{\mu_t}$ arbitrarily small, $\eta^{\!-\!1}_{\pi_{t+1}}D_{h}(\operatorname{greedy}((\mathcal{T}_{\mu_t} Q_t)(s)), \pi_t(s)) \leq \eps_{\pi_{t+1}}$ and ${\eta}^{\!-\!1}_{\mu_t} D_{h}(\operatorname{greedy}(Q_{t}), \pi_t(s)) \leq \eps_{\mu_t}$. Then
\eqq{
 V^{*}(s) -V_{t+1}(s)&\leq  \gamma^2 \| V^{*}  - V_t\|_{\infty} +\gamma\eps_{\mu_t}  +{\eps}_t\label{eq:mini_claim2}
}
and recursing yields
\eqq{
\textstyle\|V^{*}(s) -V_t(s)\|_{\infty} &\leq (\gamma^{2})^{t}( \| V^{*}  - V_0\|_{\infty} + \sum_{i\leq t} \nicefrac{(\eps_{\pi_{i+1}} + \gamma\eps_{\mu_i})}{(\gamma^2)^i} )\label{eq:recursive_claim2}
}
\end{theorem}
\begin{proof}
If $\pi_{t+1}$ is the result of a PMD update with $(\mathcal{T}_{\mu_t} Q_t)(s)$, applying Lemma~\ref{Three-Point Descent Lemma} for $\bar{\pi}_{t+1}(s)$ greedy with respect to $(\mathcal{T}_{\mu_t} Q_t)(s)$
\eqq{
-\langle (\mathcal{T}_{\mu_t} Q_t)(s) , \pi_{t+1}(s) - \bar{\pi}_{t+1}(s)\rangle
&\leq \nicefrac{1}{
\eta_{\pi_{t+1}}}\big(D_{h}(\bar{\pi}_{t+1}(s),\pi_t(s)) - D_{h}(\pi_{t+1}(s),\pi_t(s)) \nonumber
\\
&\qquad- D_{h}(\pi_{t+1}(s),\bar{\pi}_{t+1}(s)) \big)\nonumber
\\
&\leq \eta^{\!-\!1}_{\pi_{t+1}}D_{h}(\bar{\pi}_{t+1}(s),\pi_t(s)) \label{eq:final_ineq3}
}
Further, the suboptimality is
\eqq{
V^{*}(s) -V_{t+1}(s) &= -\langle  (\mathcal{T}_{\mu_t} Q_t)(s),  {\pi}_{t+1}(s) - {\pi}^*(s)\rangle -  \langle Q_{t+1}(s) - (\mathcal{T}_{\mu_t} Q_t)(s), \pi_{t+1}(s)\rangle\nonumber
\\
&\qquad
+\langle Q^{*}(s) - (\mathcal{T}_{\mu_t} Q_t)(s), \pi^*(s)\rangle \label{eq:final_eq2}
}
Since $\langle  (\mathcal{T}_{\mu_t} Q_t)(s),{{\pi}^*(s)}\rangle \leq \langle  (\mathcal{T}_{\mu_t} Q_t)(s),{\bar{\pi}_{t+1}(s)}\rangle$ if ${\bar{\pi}_{t+1}(s)}$ is greedy with respect to $(\mathcal{T}_{\mu_t} Q_t)(s)$, then plugging Eq~\ref{eq:final_ineq3} in Eq~\ref{eq:final_eq2} we have
\eqq{
V^{*}(s) -V_{t+1}(s)&\leq \eta^{\!-\!1}_{\pi_{t+1}}D_{h}(\bar{\pi}_{t+1}(s),\pi_t(s))  -  \langle Q_{t+1}(s) - (\mathcal{T}_{\mu_t} Q_t)(s), \pi_{t+1}(s)\rangle
\nonumber \\
&\qquad +\langle Q^{*}(s)- (\mathcal{T}_{\mu_t} Q_t)(s), \pi^*(s)\rangle \label{eq:final_last_eq2}
}
Next, cf. Lemma~\ref{lemma:descent property of PMD+}, 
$\langle Q_{t+1}(s) - (\mathcal{T}_{\mu_t} Q_t)(s), \pi_{t+1}(s) \rangle\geq 0$.
Plugging back into Eq~\ref{eq:final_eq2}, we obtain
\eqq{
V^{*}(s) -V_{t+1}(s)&\leq \eta^{\!-\!1}_{\pi_{t+1}}D_{h}(\bar{\pi}_{t+1}(s),\pi_t(s)) 
+\langle Q^{*}(s) - (\mathcal{T}_{\mu_t} Q_t)(s), \pi^*(s)\rangle \nonumber
\\
&\leq  \langle Q^{*}(s) - (\mathcal{T}_{\mu_t} Q_t)(s), \pi^*(s)\rangle + \eps_{\pi_{t+1}} \label{eq:to_recurs2}
}
where the last step follows from step-size adaptation condition.

Applying Lemma~\ref{Three-Point Descent Lemma} for $\tilde{\pi}_{t+1}(s)$ greedy with respect to ${Q}_{t}$, 
\eqq{
-\langle{Q}_{t}(s), \mu_t(s)-{\pi}^{*}(s)\rangle\leq-\langle {Q}_{t}(s) , \mu_t(s) - \tilde{\pi}_{t+1}(s)\rangle
&\leq \eta^{\!-\!1}_{\mu_{t}}\big(D_{h}(\tilde{\pi}_{t+1}(s),\pi_t(s)) \nonumber
\\
&\qquad- D_{h}(\pi_{t+1}(s),\pi_t(s)) \nonumber
\\
&\qquad- D_{h}(\pi_{t+1}(s),\tilde{\pi}_{t+1}(s)) \big)\nonumber
\\
&\leq \eta^{\!-\!1}_{\mu_{t}}D_{h}(\tilde{\pi}_{t+1}(s),\pi_t(s)) \label{eq:final_ineq2}
}
Further,
\eq{
\langle Q^{*}(s) - (\mathcal{T}_{\mu_t} Q_t)(s), {\pi}^{*}(s)\rangle
&= \sum_{a\in\A} (Q^{*}({s,a}) -(\mathcal{T}_{\mu_t} Q_t)({s,a}))\pi^{*}({a|s})\nonumber
\\
 &= \gamma  \sum_{s^\prime\in\S}  \sum_{a\in\A}  P({s^{\prime}|s, a}) \pi^{*}({a|s}) \left[\langle  Q^{*}({s^\prime}),\pi^{*}({s^\prime})\rangle- \langle {Q}_{t}({s^\prime}),\mu_t({s^\prime})\rangle\right]\nonumber
 \\
  &= \gamma  \sum_{s^\prime\in\S}  \sum_{a\in\A}  P({s^{\prime}|s, a}) \pi^{*}({a|s}) \left[\langle  Q^{*}({s^\prime})-  {Q}_{t}({s^\prime}),{\pi}^*({s^\prime})\rangle- \langle {Q}_{t}({s^\prime}),\mu_t({s^\prime})-{\pi}^*({s^\prime})\rangle\right]\nonumber
 \\
  &\overset{Eq.~\ref{eq:final_ineq2}}{\leq} \gamma  \sum_{s^\prime\in\S}  \sum_{a\in\A}  P({s^{\prime}|s, a}) \pi^{*}({a|s}) \left[\langle  Q^{*}({s^\prime})-  {Q}_{t}({s^\prime}),{\pi}^*({s^\prime})\rangle + \eta^{\!-\!1}_{\mu_{t}}D_{h}(\tilde{\pi}_{t+1}(s),\pi_t(s))\right]
  \\
   &\overset{\text{cf. premise}}{\leq} \gamma  \sum_{s^\prime\in\S}  \sum_{a\in\A}  P({s^{\prime}|s, a}) \pi^{*}({a|s}) \left[\langle  Q^{*}({s^\prime})-  {Q}_{t}({s^\prime}),{\pi}^*({s^\prime})\rangle + {\eps}_{\mu_t} \right]
  \\
  &= \gamma^2  \sum_{s^\prime\in\S}  \sum_{a\in\A} \pi^{*}({a|s}) P({s^{\prime}|s, a}) \left[\sum_{s^{\prime\prime}\in\S} \sum_{a^\prime\in\A} {\pi}^*({s^\prime,a^\prime})P({s^{\prime\prime} | s^\prime, a^\prime})\left[  V^{*}({s^{\prime\prime}})-  {V}_{t}({s^{\prime\prime}})\right]+{\eps}_{\mu_t}\right]
}
Taking the max norm, using the triangle inequality and contraction property
\eq{
\|\langle Q^{*} - (\mathcal{T}_{\mu_t} Q_t), \pi^*\rangle \|_{\infty} &\leq  \gamma^2\| V^{*} - {V}_{t}\ \|_{\infty} +\gamma\eps_{\mu_t}
}
Plugging back in Eq.~\ref{eq:to_recurs2}
\eq{
V^{*}(s) -V_{t+1}(s)&\leq \langle Q^{*}(s) - (\mathcal{T}_{\mu_t} Q_t)(s), \pi^*(s)\rangle + \eps_{\pi_{t+1}}
\\
&\leq  \|\langle Q^{*} - (\mathcal{T}_{\mu_t} Q_t), \pi^*\rangle \|_{\infty} + \eps_{\pi_{t+1}} 
\\
&\leq  \gamma^2 \| V^{*}  - V_t\|_{\infty}  +\gamma\eps_{\mu_t}+ \eps_{\pi_{t+1}} 
}
which is the first claim in Eq.\ref{eq:mini_claim2}. Recursing yields the second claim in Eq.\ref{eq:recursive_claim2}.
\end{proof}

\clearpage
\section{Details on PMD updates for Sec.~\ref{sec:Functional Acceleration for PMD}: Functional Acceleration for PMD}\label{append:Algorithms}

\begin{algorithm}[H]
\caption{{PMD(++)}}
\label{alg:PMD++}
\begin{algorithmic}[1]
{\footnotesize
\STATE Input: ${\mu_0},{\pi_0}\in \operatorname{rint}\Pi$, adaptive $\{{{\eta}_{\pi_{t+1}}},{{\eta}_{\mu_{t}}}\}_{t\geq 0}$
  \FOR{$t = 1,2 \dots T$}
\STATE\underline{\texttt{PMD(+lookahead)}}
  \eq{
  { \mu_t(s)}&=\operatorname{argmin}_{\pi(s)\in \Delta(\A)}-\langle {Q_t(s)}, \pi(s)\rangle, \quad (\mathcal{T}_{\mu_t} Q_t)(s) \doteq \E[r(s) + \gamma \langle {Q}_{t}({s^\prime}), \mu^t({s^\prime})\rangle]
  \\
  {{\pi}_{t+1}(s)}&=\operatorname{argmin}_{\pi(s)\in \Delta(\A)}-\langle 
         {(\mathcal{T}_{\mu_t} Q_t)(s)}, \pi(s)\rangle + \eta^{\!-\!1}_{\pi_{t+1}} D_{{h}}(\pi(s), {{\pi}_t(s)})
         }
    \STATE\underline{\PMDext}
     \eq{
     { \mu_t(s)}&=\operatorname{argmin}_{\pi(s)\in \Delta(\A)}-\langle { Q_t(s)}, \pi(s)\rangle + \eta^{\!-\!1}_{\mu_t}D_{{h}}(\pi(s), {\pi_t(s)}), \quad (\mathcal{T}_{\mu_t} Q_t)(s) \doteq \E[r(s) + \gamma \langle {Q}_{t}({s^\prime}), \mu^t({s^\prime})\rangle]
         \\
         { {\pi}_{t+1}(s)}&=\operatorname{argmin}_{\pi(s)\in \Delta(\A)}-\langle {(\mathcal{T}_{\mu_t} Q_t)(s)}, \pi(s)\rangle +{\eta}^{\!-\!1}_{\pi_{t+1}} D_{{h}}(\pi(s), {{\pi}_t(s)})
         }
          \STATE\underline{\PMDmom} 
           \eq{
           {{\pi}_{t+1}(s)}&= \underset{\pi(s)\in \Delta(\A)}{\operatorname{argmin}}-\langle {{Q}_t(s)} + {\eta}^{\!-\!1}_{\pi_{t+1}}{\eta}_{\pi_{t}} ({{Q}_t(s)} - {{Q}_{t-1}(s)}), \pi(s)\rangle + {\eta}^{\!-\!1}_{\pi_{t+1}} D_{{h}}(\pi(s), {{\pi}_t(s))}
           }
 \ENDFOR
    }
\end{algorithmic}
\end{algorithm}

\section{Details on Algorithmic Implementation for Sec.~\ref{sec:Approximate Functional Acceleration for Parametric Policies}: Approximate Functional Acceleration for Parametric Policies}\label{append:Algorithmic Implementations}\label{append:Details for Approximate Functional Acceleration for Parametric Policies}
There are two ways of updating the parameter vector $\theta$ (cf. Lemma~\ref{lemma:Proximal perspective on mirror descent} \citep{Bubeck15}): (i) the PGD perspective of MD \citep{alfano2024novel, Haarnoja18, Abdolmaleki18} (see Appendix~\ref{append:generalizedGD}),
or (ii) the proximal perspective \citep{Tomar20, Vaswani2021, vaswani2023decisionaware}. The latter is used and described in Alg.~\ref{alg:PMD+}.

We execute the parameter optimization in Alg.~\ref{alg:PMD+} in expectation over the state-action space---in full-batch (computing $d^{\pi}_\rho$ exactly and in expectation for all actions) to showcase the higher-level optimization that is the spotlight of this work and remove any other collateral artifacts or confounding effects from exploration  of the state space or too early committal to a strategy \citep{Mei2021}. Practical large-scale algorithms apply mini-batches sampled from a reply buffer, with the updates somewhere between full-batch and online. In making this simplification, we inevitably leave complementary investigations on the influence of stochasticity and variance of the policy gradient for future work. 



\textbf{Policy approximation}\quad We parametrize the policy iterates using a Bregman policy class 
$\{\pi_{\theta}:\pi_{\theta}(s) = \operatorname{proj}^h_{\Delta(\A)}(\nabla h^*(f_\theta(s))), s \in \S\}$ with a tabular parametrization $\theta$. For the updates requiring two policies, we keep them parametrized separately with $\mu_\w$ and $\pi_\theta$. 
We formulate the policy optimization problem using the extension proposed by \citet{Tomar20}. Each iteration, in an ``inner-loop'' optimization procedure, we update $\theta$ and $\w$ using $k$ updates with standard GD on the composite PMD surrogate model, denoted $\ell: \Theta  \to \R$ (with $\Theta\doteq \R^{|\S|\times|\A|}$ cf. the tabular parametrization) associated with the policy represented by those parameters $\pi_\theta$, or $\mu_\w$, respectively. We execute the parameter optimization in expectation over the state-action space. Concretely, for \PMD\ we use the surrogate 
\eqq{
\ell(\theta) 
\doteq 
 \E_{s\sim d^{\rho}_t} \left[-\langle {Q}_{\pi_t}(s), \pi_{\theta}\rangle \!+\! \eta^{\!-\!1}_{\pi_{t+1}}(s)  D_{{h}}(\pi_{\theta}(s), \pi_t(s))\right]\label{eq:pmd_obj_params}
 }
  and update in an ``inner-loop'' optimization procedure
\eq{
\text{\emph{(init)}}\ &\theta^{(0)} \doteq \theta_t \qquad \text{\emph{(for $i \in [0..k]$)}}\ \theta^{(i+1)} = \theta^{(i)} - \beta \nabla_{\theta^{(i)}} \ell(\theta^{(i)})\qquad\text{\emph{(final)}}\ \theta_{t+1}  \doteq \theta^{(k)}
}
 with  $\beta$---a small learning rate.
 The optimization procedure for $\w$ is analogous. The rest of the algorithms use the surrogate objectives as described in Sec.~\ref{sec:Functional Acceleration for PMD}. 

  \paragraph{Value function}
  Everywhere in the numerical studies, we used ${Q}_{\mu_t}$ instead of $(\mathcal{T}{Q}_{\pi_t})(s) \doteq \E[r({s}) + \gamma \langle {Q}_{\pi_t}({s^\prime}), \mu_t\rangle]$.
  In the inexact setting, we estimate $\widehat{Q}_{\pi_{t}}$ and use it in place of ${Q}_{\pi_{t}}$. Similarly, $\widehat{Q}_{\mu_t}$ in place of ${Q}_{\mu_t}$.
 
 \paragraph{Step-size adaptation} We use state dependent step-sizes, and for step-size adaptation, we compute $\eta_{\mu_t}(s) = {\gamma^{\!-\!t}\eps^{\!-\!1}_0}{ D_{\Omega}(\operatorname{greedy}({Q}_{\pi_t}(s)), \pi_t(s))}$, with $\eps_0 = 10^{-4}$ a small constant according to the optimal adaptive schedule for PMD, cf. \citet{johnson2023optimal}. The value for $\eta_{\pi_{t+1}}(s)$ is chosen analogously.

\paragraph{Objectives}
We now describe in detail the objectives for each algorithm.
 
\underline{\textbf{\PMD}}---We approximate the  optimization of the objective in Eq.\ref{eq:pmd_obj_params} with respect to the parameters  ${\theta}$ of ${\pi_{\theta}}$ with ${k}$ GD  updates using a per-state step-size ${\eta_{\mu_t}(s)} \!=\! {\gamma^{\!-\!t}\eps^{\!-\!1}_0}{D_{\Omega}(\operatorname{greedy}({Q}_{\pi_t}(s)), {\pi_t(s)})}$.

\underline{\textbf{\pmdloo}}---We keep two policies ${\mu}$, ${\pi_{\theta}}$---the former is the non-parametric greedy policy 
$${\mu_{t}(s)} = \max_{a\in\A}  {Q}_{\pi_t}(s)(a)$$
The latter is parametrized and its parameters ${{\theta}}$ optimize the following lookahead-based surrogate with $k$ GD updates using step-size adaptation ${\eta_{\pi_{t+1}}(s)}\! =\! {\gamma^{\!-\!t}\eps^{\!-\!1}_0}{D_{\Omega}(\operatorname{greedy}((\mathcal{T}_{\mu_t}Q_{\pi_t})(s)), {\pi_t(s)})}$.
\eq{
\ell(\theta) 
\doteq 
 \E_{s\sim d^{\rho}_t} \left[-\langle (\mathcal{T}_{\mu_t}Q_{\pi_t})(s), \pi_{\theta}(s)\rangle \!+\! \eta^{\!-\!1}_{\pi_{t+1}}(s) D_{{h}}(\pi_{\theta}(s), \pi_t(s))\right]\label{eq:pmd_obj_params}
 }
In the results for the experimental section, we used ${Q_{\mu_t}}$ in place of $\mathcal{T}_{\mu_t}Q_{\pi_t}$.

\underline{\textbf{\pmdext}}---The update to ${\pi_{\theta}}$ is identical to {\pmdloo}. In  contrast to \pmdloo, ${\mu_{\w}}$ is  parametrized with parameter vector ${\w}$. The update to ${\mu_{\w}}$ uses, at each iteration, $k$ GD updates on the surrogate objective
\eq{
\ell(\w) 
&\doteq 
 \E_{s\sim d^{\rho}_t} \left[-\langle {Q}_{\pi_t}(s), {\mu_{\w}(s)}\rangle \!+\! \eta^{\!-\!1}_{\mu_{t}} D_{\Omega}({\mu_{\w}(s)}, {\pi_{t}(s)}) \right] 
}
The step-sizes ${\eta_{\mu_t}}$ are adapted using ${\eta_{\mu_t}(s)} = {\gamma^{\!-\!t}\eps^{\!-\!1}_0}{D_{\Omega}(\operatorname{greedy}({Q}_{\pi_t}(s)), {\pi_t(s)})}$.



\underline{\textbf{\pmdmom}}---A single set of parameters ${\theta}$ are learned by updating, at each iteration, using $k$ GD updates on the objective
\eq{
\ell(\theta) 
\doteq \E_{s\sim d^{\rho}_t} \left[ -\langle {{Q}_{\pi_t}(s)} +\eta^{\!-\!1}_{\pi_{t+1}}\eta_{\pi_{t}}[{{Q}_{\pi_t}(s)} - {{Q}_{\pi_{t-1}}(s)}], {\pi_{\theta}(s)}\rangle + \eta^{\!-\!1}_{\pi_{t+1}}(s) D_{\Omega}({{\pi}_{\theta}(s)}, {{\pi}_{t}(s)})
\right] 
}

\section{Approximate Policy Mirror Descent as Projected Gradient Descent (PGD)}\label{append:generalizedGD}
In this section we provide an alternative perspective on PMD---cf. Lemma~\ref{lemma:Proximal perspective on mirror descent}, stating that
the MD update can be rewritten in the following ways
\eq{
x_{t+1} &= \argmin_{x\in\mathcal{X}\cap \mathcal{C}}D_h(x,\nabla h^*(\nabla h(x_t) + \eta \nabla f(x_t)))&&\text{(PGD)}
\\
 &= \argmin_{x\in\mathcal{X}\cap \mathcal{C}} \eta \left\langle \nabla f(x_t), x\right\rangle + D_h(x, x_t) &&\text{(proximal perspective)}
}
\citet{alfano2024novel} introduces the concept of Bregman policy class $\{\pi_{\theta}:\pi_{\theta}(s) = \operatorname{proj}^h_{\Delta(\A)}(\nabla h^*(f_\theta(s))), s \in \S\}$, and uses a parametrized function $f_{\theta}$ to approximate the dual update of MD $f_{t+1}(s) \doteq \nabla h(\pi_t(s)) - \eta_{\pi_{t+1}} \widehat{Q}_t(s)$. To satisfy the simplex constraint, \citet{alfano2024novel} uses the L2 norm to measure function approximation errors, whereas \citet{xiong2024dualapproximationpolicyoptimization} extends this to a Bregman projection on the dual approximation mapped back to the policy space 
$\pi_\theta(s) = \operatorname{proj}^{h}_{\Delta(\A)} (\nabla h^*(f_{t+1}(s)))$, equivalent to 
$\theta_{t+1} = \argmin_{\theta \in \Theta}
D_{h}(\pi_{\theta}(s), \nabla h^*(f_{t+1}(s)))$ (cf. \citet{Amari2016}, the divergences derived from two convex functions are substantially the same, except for the order of the arguments $D_{h}(x, \nabla h^*(\nabla h(y))) = D_{h^*}(\nabla h(y), \nabla h^*(x))$).

 Using the negative Boltzmann-Shannon entropy, yields the softmax policy class
 ${\pi}_{\theta}({s,a}) \doteq \nicefrac{\exp {\color{bblue}f}_\theta({s,a})}{\|\exp {\color{bblue}f}_\theta({s})\|_1}, \forall s,a\in\S\times\A$. 
 
 For the approximate setting, in the main text, we rely on the formulation introduced by \citet{Tomar20}. Here, we  provide additional details on the formulation by \citet{alfano2024novel}  and \citet{xiong2024dualapproximationpolicyoptimization}.

\begin{algorithm}[H]
\caption{{\textbf{Approximate \pmdmom\ (PGD perspective)}}}
\label{alg:PMD_genGD}
\begin{algorithmic}[1]
{\footnotesize
\STATE Initialize policy parameter 
$\theta_0 \in \Theta$, mirror map $h$, small constant $\eps_0$, learning rate $\beta$
  \FOR{$t = 1,2 \dots T$}
\STATE Approximate ${Q}_{\pi_t}$ with $\widehat{Q}_{t}$ (critic update)
\STATE Compute adaptive step-size $\eta_{\pi_{t+1}}(s) = \nicefrac{D_{h}(\operatorname{greedy}(\widehat{Q}_t(s))}{\gamma^{2(t+1)}\eps_0}$
\STATE  Find $\pi_{t+1} \doteq \pi_{\theta_{t+1}} = \nabla h^*(f_{\theta_{t+1}})$ by (approximately) solving the surrogate problem (with $k$ GD updates)
\STATE \qquad 
$\textstyle\operatorname{min}_{\theta\in \Theta}\ell(\theta)\qquad \ell(\theta)\doteq -\E_{s\sim d^{\rho}_{t}}[ D_{h}(\pi_{\theta}(s), \nabla h^*(\nabla h(\pi_t(s)) - \eta_{\pi_{t+1}}(s) \widehat{Q}_t(s) - {\eta}_{\pi_{t}}(s) (\widehat{Q}_t(s) -\widehat{Q}_{t-1}(s)))] $
\STATE  \qquad $\text{\emph{(init)}}\ \theta^{(0)} \doteq \theta_t \qquad \text{\emph{(for $i \in [0..k\!-\!1]$)}}\ \theta^{(i+1)} = \theta^{(i)} - \beta \nabla_{\theta^{(i)}} \ell(\theta^{(i)})\qquad\text{\emph{(final)}}\ \theta_{t+1}  \doteq \theta^{(k)}$
 \ENDFOR
    }
\end{algorithmic}
\end{algorithm}
Alg.~\ref{alg:PMD_genGD} describes a PGD perspective on \pmdmom, following \citet{alfano2024novel,xiong2024dualapproximationpolicyoptimization, Haarnoja18, Abdolmaleki18}.
At each iteration, after taking a gradient step in dual space, a Bregman projection is used on the dual approximation mapped back to the policy space, to satisfy the simplex constraint. 
 
\section{Newton’s method}\label{Newton’s method}
The Newton-Kantorovich theorem generalizes Newton's method for solving nonlinear equations to infinite-dimensional Banach spaces.
It provides conditions under which Newton's method converges and gives an estimate of the convergence rate. Newton-Kantorovich theorem deals with the convergence of Newton's method for a nonlinear operator $F:\mathcal{X} \to \mathcal{Y}$, where $\mathcal{X}$ and $\mathcal{Y}$ are Banach spaces. The method iteratively solves $F(x) = 0$ using
\eq{
x_{t+1} = x_t - (\nabla F)^{-1} F(x_t)
}
where $\nabla F$ is a generalization of the Jacobian of $F$, provided $F$ is differentiable.
Intuitively, at each iteration, the method performs a linearization of $F(x) = 0$ close to $x$, using a first order Taylor expansion:
$F(x + \Delta x) \approx F(x) + \nabla F(x) \Delta x$, where $F(x) + \nabla F(x) \Delta x = 0 \iff \Delta x =-(\nabla F)^{-1}F(x)$. For $x$ close to $x^*$, $x -(\nabla F)^{-1}F(x)$ is a good approximation of $x^*$. The iterative sequence $\{x_t\}_{t\geq 0}$ converges to $x^*$, assuming the Jacobian matrix exists, is invertible, and Lipschitz continuous.

\paragraph{Quasi-Newton methods}
Any method that replaces the exact computation of the Jacobian matrices in the Newton's method (or their inverses) with an approximation, is a quasi-Newton method. A quasi-Newton method constructs a sequence of iterates $\{x_t\}_{t\geq 0}$
and a sequence
of matrices $\{J_t\}_{t\geq 0}$
such that $J_t$ is an approximation of the Jacobian $\nabla F(x_t)$
for any $t\geq 0$ and
\eq{
x_{t+1} = x_t - (J_t)^{-1} F(x_t)
}

In Anderson's acceleration \citep{Anderson65}, information about the last iterates is used to update the approximation of $J_t$.

\paragraph{Policy iteration as Newton's method}
In the context of Markov Decision Processes (MDPs), policy iteration may be interpreted as Newton's method with the following notations and analogies. First, using the Bellman  optimality operator $\mathcal{T} V(s) \doteq\max_a [r({s,a}) + \gamma \sum_{s^\prime\in\S} P({s^\prime|s,a})V({s^\prime})]$, the aim is to find $V$ such that $V = \T V$, which is akin to finding the roots $V$, such that $F(V) = V- \T V = (I-\T)(V) =  0$. We interpret $F = \nabla f$ as the gradient of an unknown function $f:\R^n \to \R^n$, despite the Bellman operator being non-differentiable in general due to the max. Where the greedy policy ${\pi}_t$ attains the max in $\T V$, we obtain $J_t = \I - \gamma \mathbf{P}_{{\pi}_{t}}$, which is invertible for $\gamma \in [0,1)$.
Expanding the Bellman operator, we have
\eq{
V_{\pi_{t+1}} = r_{\pi_{t+1}} + \gamma P_{\pi_{t+1}} V_{\pi_{t+1}} \implies V_{\pi_{t+1}} &= (J_t)^{-1} r_{\pi_{t+1}}
\\
J_t V_{\pi_{t+1}} &=r_{\pi_{t+1}}
}
The values corresponding to two successive PI steps can be related in the following way by manipulating the equations  \citep{Puterman1979, grandclément2021convex}
\eq{
V_{\pi_{t+1}} & =(J_t)^{-1}r_{\pi_{t+1}} \\
& =V_{\pi_t}-V_{\pi_t}+(J_t)^{-1}r_{\pi_{t+1}} 
\\
& =V_{\pi_t}-(J_t)^{-1}J^{t} V_{\pi_t}+(J_t)^{-1}r_{\pi_{t+1}} 
\\
& =V_{\pi_t}-(J_t)^{-1}(-r_{\pi_{t+1}}+J^{t} V_{\pi_t})
\\
& =V_{\pi_t}-(J_t)^{-1}(-r_{\pi_{t+1}}+(I-\gamma P_{\pi_{t+1}}) V_{\pi_t})
\\
& =V_{\pi_t}-(J_t)^{-1}(V_{\pi_t}-r_{\pi_{t+1}}-\gamma P_{\pi_{t+1}} V_{\pi_t})
\\
& =V_{\pi_t}-(J_t)^{-1}(V_{\pi_t}-\T V_{\pi_t})
}
In the main text we used the notation $\Psi_t \doteq (I - \gamma P_{\pi_t})^{-1} = (J_t)^{-1}$, and applied the definition $\nabla f (V_{\pi_t}) \doteq F(V_{\pi_t}) = (I - \T) (V_{\pi_t}) = V_{\pi_t}-\T V_{\pi_t}$ which yielded the expression
\eq{
V_{\pi_{t+1}} = V_{\pi_t} - \Psi \nabla f (V_{\pi_t})
}
\clearpage

\section{Experimental details for Sec.~\ref{sec:Numerical studies}: Numerical Studies}\label{Experimental details for Numerical illustrations}
\subsection{Details of two-state Markov Decision Processes}\label{append:Details of two-state/action Markov Decision Processes}
In this section we give the specifics of the two-state MDPs presented in this work. We make use of the notation
\eq{
 {P}(s_k | s_i, a_j) = \mathbf{P}[i\times|\A| + j][k], \text{ with } \mathbf{P} \in \R^{|\S||\A|\times|\S|}
 \\
 {r}(s_i, a_j) = \mathbf{r}[i\times|\A| + j], \text{ with } \mathbf{r} \in \R^{|\S||\A|}
}
In the main text, we used example (i), and in Appendix~\ref{append:Supplementary results for Policy dynamics in Value Space}, additionally examples (ii), (iii), and (iv)
\eq{
\text{(i)} \qquad
&|\A| = 2, \gamma = 0.9,
\mathbf{r} = [-0.45, -0.1, 0.5, 0.5],
\\
 &\mathbf{P} =
[[-0.45, 0.3],
[0.99, 0.01],
[0.2, 0.8],
[0.99, 0.01]]
\\
\\
\text{(ii)} \qquad
&|\A| = 2, \gamma = 0.9,
\mathbf{r} = [0.06, 0.38, -0.13, 0.64], 
\\
&\mathbf{P} =
[[0.01, 0.99], [0.92, 0.08], [0.08, 0.92], [0.70, 0.30]]
\\
\\
\text{(iii)} \qquad
&|\A| = 2, \gamma = 0.9,
\mathbf{r} = [0.88, -0.02, -0.98, 0.42], 
\\
&\mathbf{P} =
[[0.96, 0.04], [0.19, 0.81], [0.43, 0.57], [0.72, 0.28]]
\\
\\
\text{(iv)} \qquad
&|\A| = 3, \gamma = 0.8,
\mathbf{r} = [-0.1, -1., 0.1, 0.4, 1.5, 0.1], 
\\
&\mathbf{P} =
[[0.9, 0.1], [0.2, 0.8], [0.7, 0.3], [0.05, 0.95], [0.25, 0.75], [0.3, 0.7]]
}

\subsection{Details of Random Markov Decision Processes}\label{append:Details of Random Markov Decision Processes}
We consider randomly constructed finite MDPs---\textbf{Random MDP} problems
(a.k.a. Garnet MDPs: Generalized Average Reward Non-stationary Environment Test-bench) 
\citep{Archibald1995, Bhatnagar2009}, abstract, yet representative of the kind of MDP encountered in practice, which serve as a test-bench for RL algorithms \citep{goyal2021firstorder, ScherrerGeist2014, Vieillard2019}. A \textbf{Random MDP} generator $\mathcal{M} \doteq (|\S|, |\A|, b, \gamma)$ is parameterized by $4$ parameters: number of states $|\S|$, number of actions $|\A|$, branching factor $b$ specifying how many possible next states are possible for each state-action pair. 

The transition probabilities $P(s_0|s, a)$ are then computed as follows. First, $b$ states ($s_1, \dots s_{b}$) are chosen uniformly at random and transition probabilities are set by sampling uniform random $b \!-\! 1$ numbers (cut points) between $0$ and $1$ and sorted as $(p_0 = 0, p_1, \dots p_{b-1}, p_b =1)$. Then, the transition probabilities are assigned as $P(s_i| s, a) = p_i - p_{i-1}$ for each $1\leq i\leq b$. The reward is state-dependent, and for each MDP, the per-state reward $r(s)$ is uniformly sampled  between $0$ and $R_{\operatorname{max}}$, such that  $r \sim (0, R_{\operatorname{max}})^{|\S|}$. The illustrations shown use $|\S|=100$ and $R_{\operatorname{max}}=100$. Other choices yield similar results.

\subsection{Details of Experimental Setup for Sec.~\ref{sec:When is acceleration possible?}}\label{append:Details of Experimental Setup for When is acceleration possible?}
We use $\beta=0.5$---the learning rate of the parameter ``inner-loop'' optimization problem, $\pi_0: \emph{center}$ for all experiments of this section. We use vary one parameter of the problem while keeping all others fixed cf. Table~\ref{table:when?}. We use $50$ randomly generated MDPs for each configuration and compute the mean and standard deviation shown in the plots.

\begin{table}[ht]
\centering
{\footnotesize
\caption{The parameters used for the optimization in Sec.~\ref{sec:When is acceleration possible?}.}
\label{table:when?} 
\catcode`,=\active
\def,{\char`,\allowbreak}
\renewcommand\arraystretch{1.2}
\begin{tabular}{p{3.5cm}<{\raggedright} p{1.cm} p{3.cm}<{\raggedright} }
  \toprule
  Experiment &
    Alg/MDP param           & Values                   \\ 
  \midrule
    \emph{$k$ sweep}
    Fig.~\ref{fig:rmdp_learning_curves}
    \emph{(Left)} 
    & $k$                  & \{1, 5, 10, 20, 30\} \\
    & $b$                  & 5 \\
    & $\gamma$              & 0.95\\
    & $|\A|$                  & 10 \\
    & $T$ & 10 \\
    \midrule
         \emph{$b$ sweep}
         Fig.~\ref{fig:rmdp_learning_curves}
         \emph{(Center-Left)}  
         & $k$                  & 100 \\
         & $b$                  & \{5, 10, 20, 30, 40\} \\
          & $\gamma$              & 0.95\\
    & $|\A|$                  & 10 \\
     & $T$ & 10 \\
    \midrule
           \emph{$\gamma$ sweep}
           Fig.~\ref{fig:rmdp_learning_curves}
           \emph{(Center-Right)} 
             & $k$                  & 30 \\
          & $b$              & 5\\
          & $\gamma$                  & \{0.98, 0.95, 0.9, 0.85\} \\
    & $|\A|$                  & 10 \\
     & $T$ & 10 \\
    \midrule
   \emph{$|\A|$ sweep}
   Fig.~\ref{fig:rmdp_learning_curves}
   \emph{(Right)} 
    & $k$                  & 100 \\
         & $b$              & 5\\
          & $\gamma$                  & 0.95 \\
          & $|\A|$                  & \{2, 5, 10, 15\} \\
    & $T$ & 20 \\
    \midrule
  \bottomrule
\end{tabular} 
}
\end{table}


\clearpage
\clearpage
\section{Supplementary Results for Sec.~\ref{sec:Numerical studies}: Numerical Studies}\label{Supplementary Results for Numerical illustrations}

\subsection{Supplementary results for Sec.~\ref{sec:When is acceleration possible?}}\label{append:Supplementary results for When is acceleration possible?}
This section presents additional results to those in Sec.~\ref{sec:When is acceleration possible?}. Fig.~\ref{fig:rmdp_learning_curves_final} shows the final performance and is analogous to Fig.\ref{fig:rmdp_learning_curves} in the main text. Fig.~\ref{fig:rmdp_learning_curves_all} shows the optimality gap while learning for $T$ iterations.

\begin{figure}[h]
        \centering
         \hspace{-10pt} 
         \includegraphics[width=1.\textwidth]{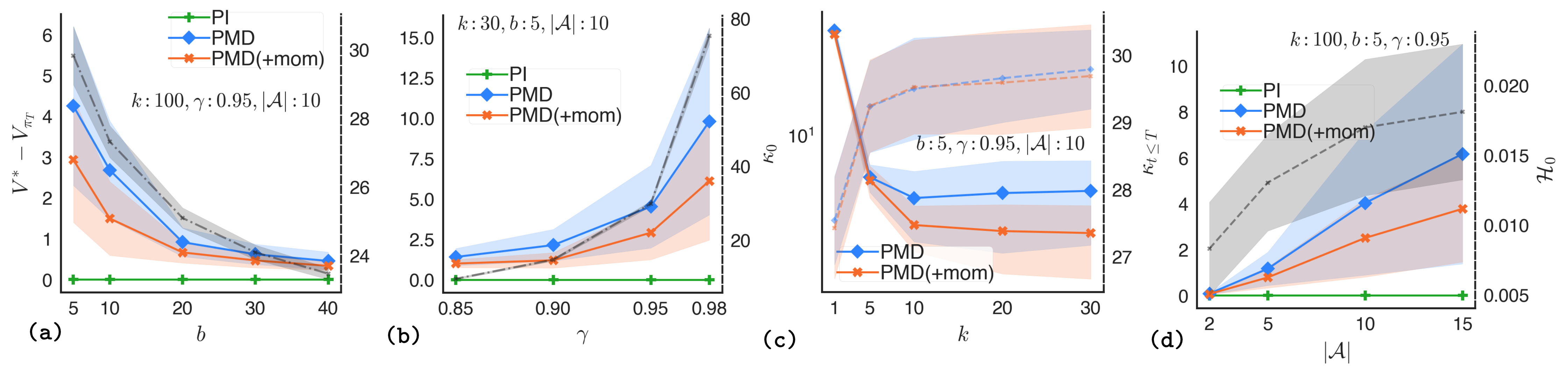}
    \caption{
    The \emph{left $y$-axis} shows the final optimality gap (regret) at timestep $T$ (cf. Table~\ref{table:when?}) of the updates: \PI, \PMD\ and \pmdmom, after $T$ iterations ($T=10$ \emph{(a-c)}, $T=20$ \emph{(d)}) relative to changing the hyperparameters: (a) $b$---the branching factor of the \emph{Random MDP}, (b) $\gamma$---the discount factor, (c) $k$---the number of parameter updates, (d) $|\mathcal{A}|$---the number of actions. Shades denote standard deviation over $50$ sampled MDPs. The \emph{right $y$-axis} and dotted curves measure: \emph{(a-b)}---the condition number $\kappa_0$, \emph{(c)} the average  condition number $\kappa_{t\leq T}$, \emph{(d)} the entropy $\mathcal{H}_0$. 
    }
    \label{fig:rmdp_learning_curves_final}
    \vspace{-10pt}
\end{figure}

\begin{figure}[h]
        \centering
         \hspace{-10pt} 
         \includegraphics[width=0.8\textwidth]{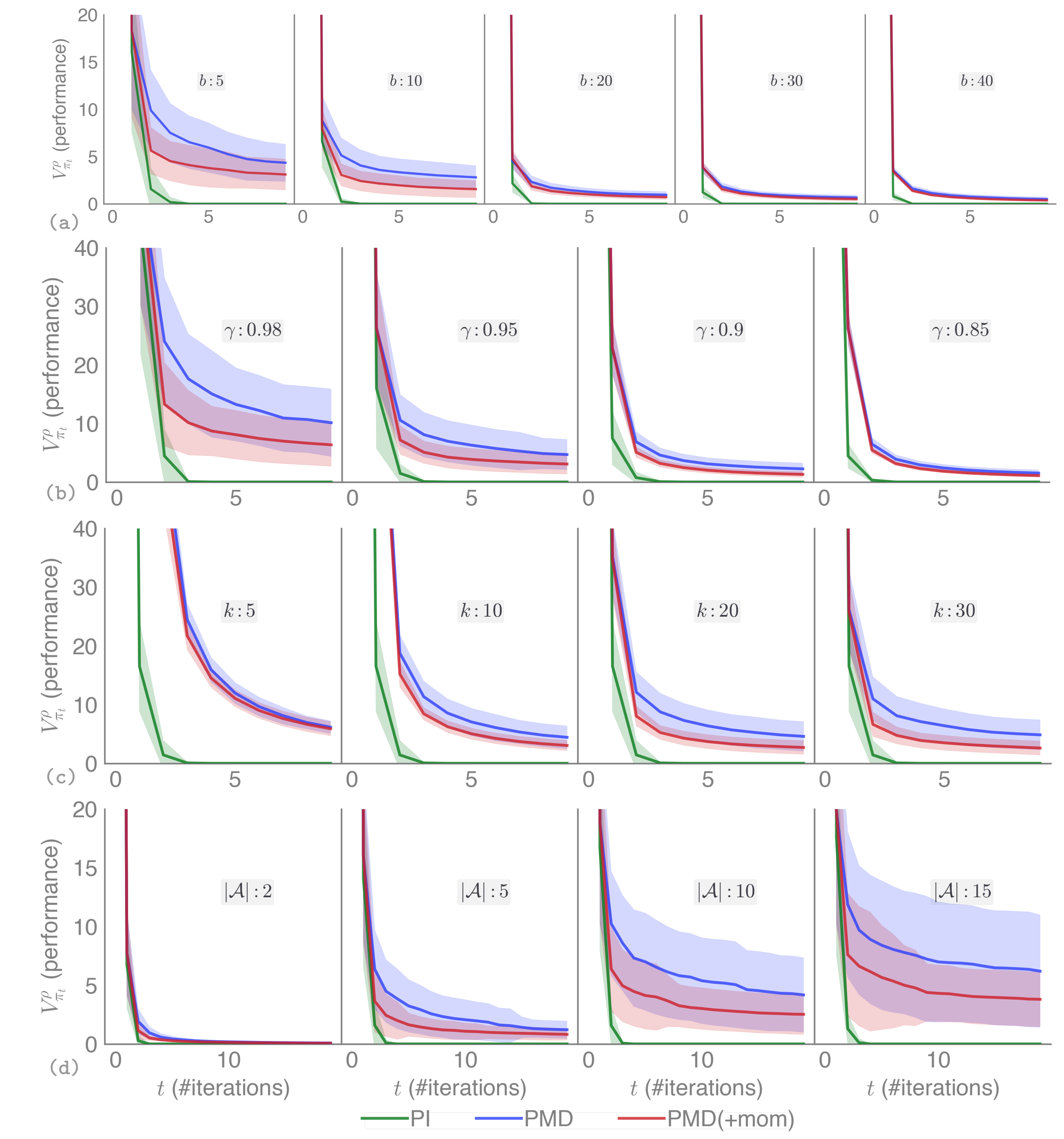}
    \caption{
    The \emph{left $y$-axis} shows the optimality gap (regret) of the updates: \PI, \PMD\ and \pmdmom, for $T$ iterations ($T=20$ final column, $T=10$ rest of the columns) relative to changing the hyperparameters: (a) $b$---the branching factor of the \emph{Random MDP}, (b) $\gamma$---the discount factor, (c) $|\mathcal{A}|$---the number of actions, (d) $k$---the number of parameter updates. Shades denote standard deviation over $50$ sampled MDPs. 
    }
    \label{fig:rmdp_learning_curves_all}
    \vspace{-10pt}
\end{figure}

\subsection{Supplementary results for Sec.~\ref{sec:Policy dynamics in Value Space}}
\label{append:Supplementary results for Policy dynamics in Value Space}
In this section we provide supplementary results that were omitted in the main body, related to the policy optimization dynamics of the functional acceleration methods introduced.

\begin{figure}[h]
        \centering
         \hspace{-10pt} 
        \includegraphics[width=1.\textwidth]{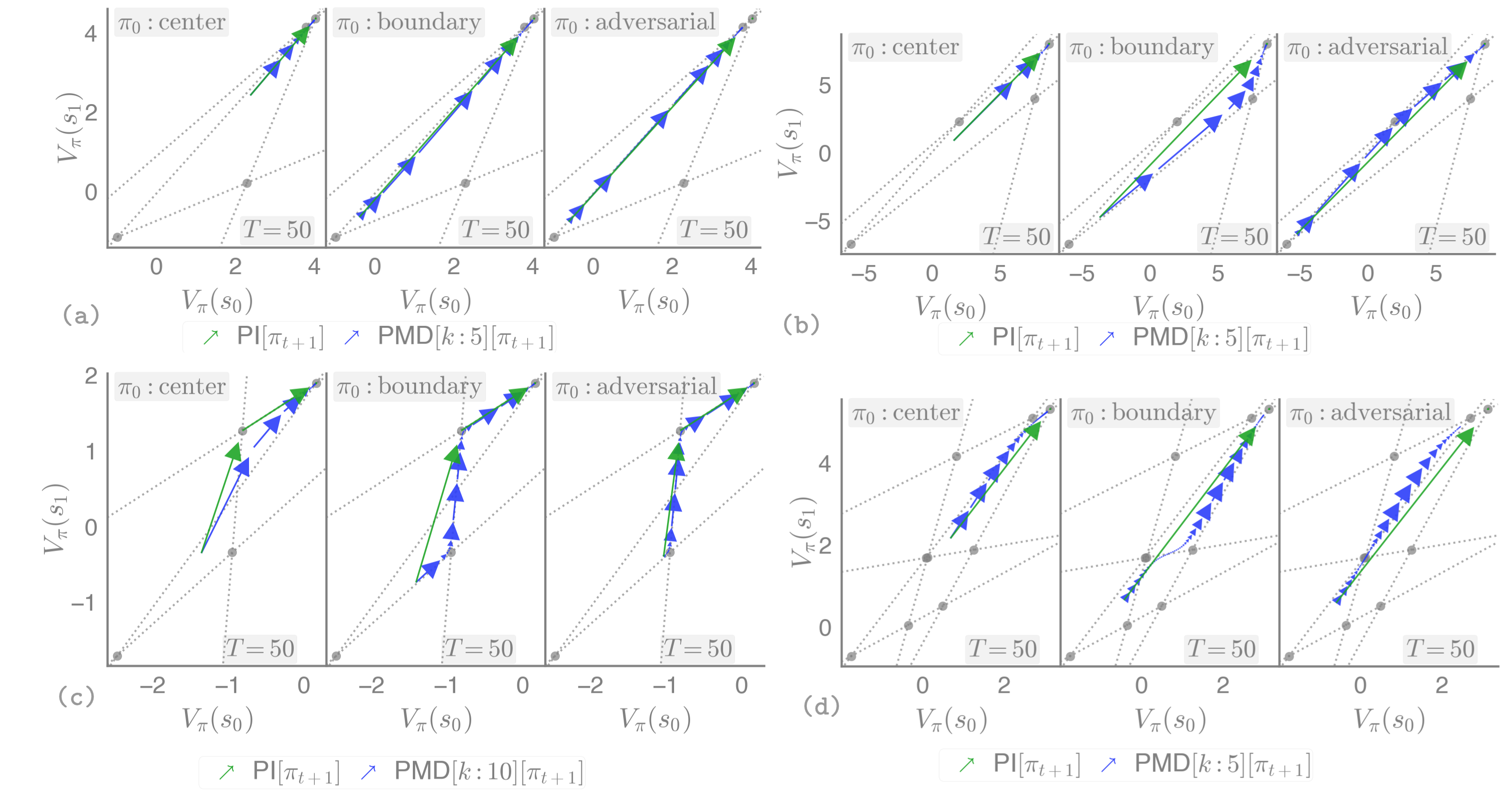}
    \caption{
    Compares the policy optimization dynamics of \PMD\ and \PI\ on the value polytope of the different example MDPs in Sec.~\ref{append:Details of two-state/action Markov Decision Processes}: \emph{(a)} example (ii), \emph{(b)} example (iii), \emph{(c)} example (i), \emph{(d)} example (iv). 
    }
    \label{fig:suppl_how__pi_pmd}
    \vspace{-10pt}
\end{figure}
\begin{figure}[h]
        \centering
         \hspace{-10pt} 
        \includegraphics[width=1.\textwidth]{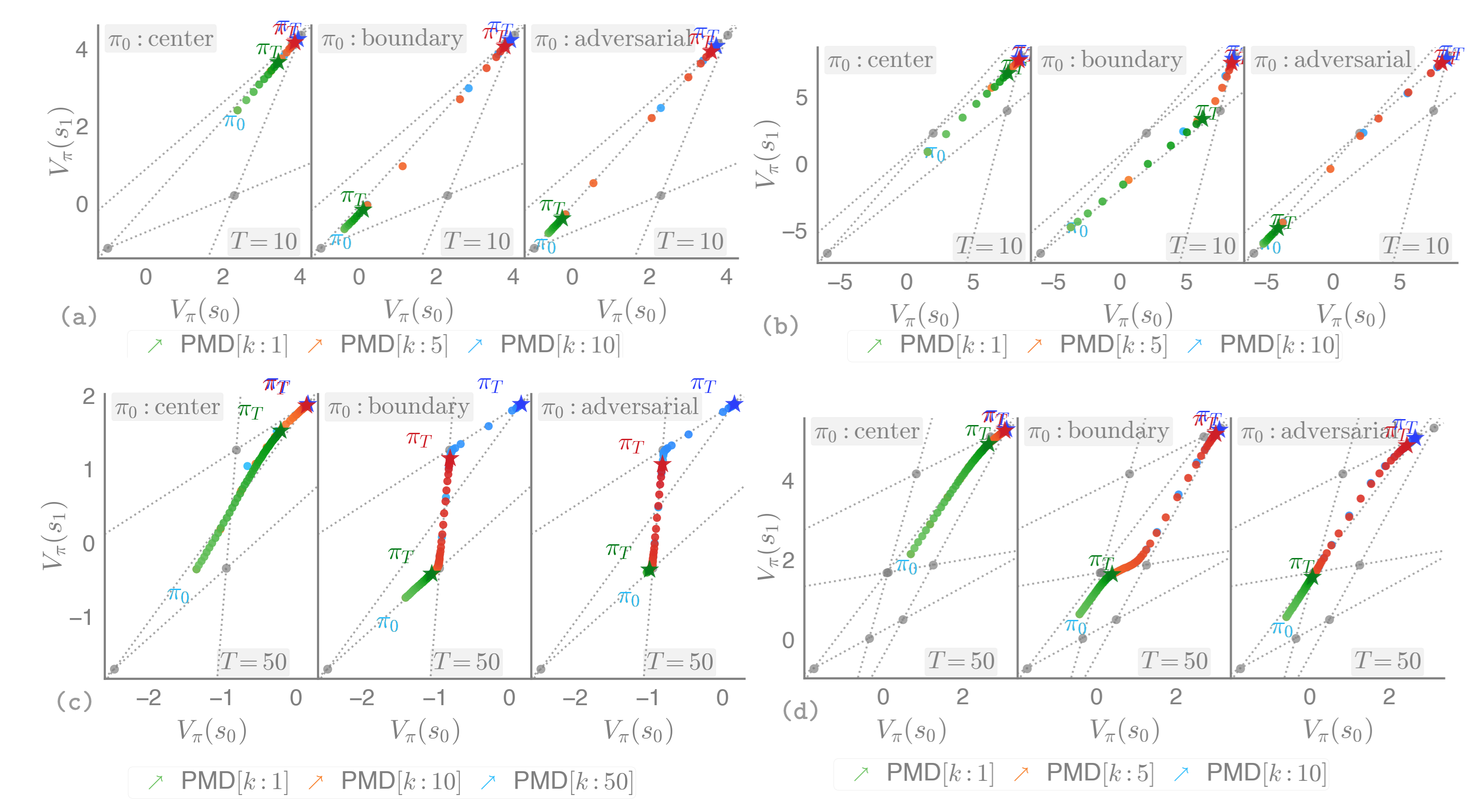}
    \caption{
    Shows the policy optimization dynamics of \PMD\ for different values of $k$ on the value polytope of the different example MDPs in Sec.~\ref{append:Details of two-state/action Markov Decision Processes}: \emph{(a)} example (ii), \emph{(b)} example (iii), \emph{(c)} example (i), \emph{(d)} example (iv). 
    }
    \label{fig:suppl_how__pmd_sweep_k}
    \vspace{-10pt}
\end{figure}
\paragraph{\PMD}
{
In Fig.~\ref{fig:suppl_how__pi_pmd}, we compare the optimization dynamics of \PMD\ and \PI\ for different MDPs. We observe the policy tends to move in a straight line between semi-deterministic policies (cf. \citet{dadashi2019value}) on the boundary of the polytope, and when it passes over an attractor point it can get delayed slowing down convergence. Fig.~\ref{fig:suppl_how__pmd_sweep_k} shows the speed of convergence is governed by $k$ which reflects the inner-loop optimizationn procedure. We again observe in Fig.~\ref{fig:suppl_how__pmd_sweep_k} the accummulation points and long-escape attractor points of the optimization procedure.
}

\begin{figure}[h]
        \centering
         \hspace{-10pt} 
        \includegraphics[width=1.\textwidth]{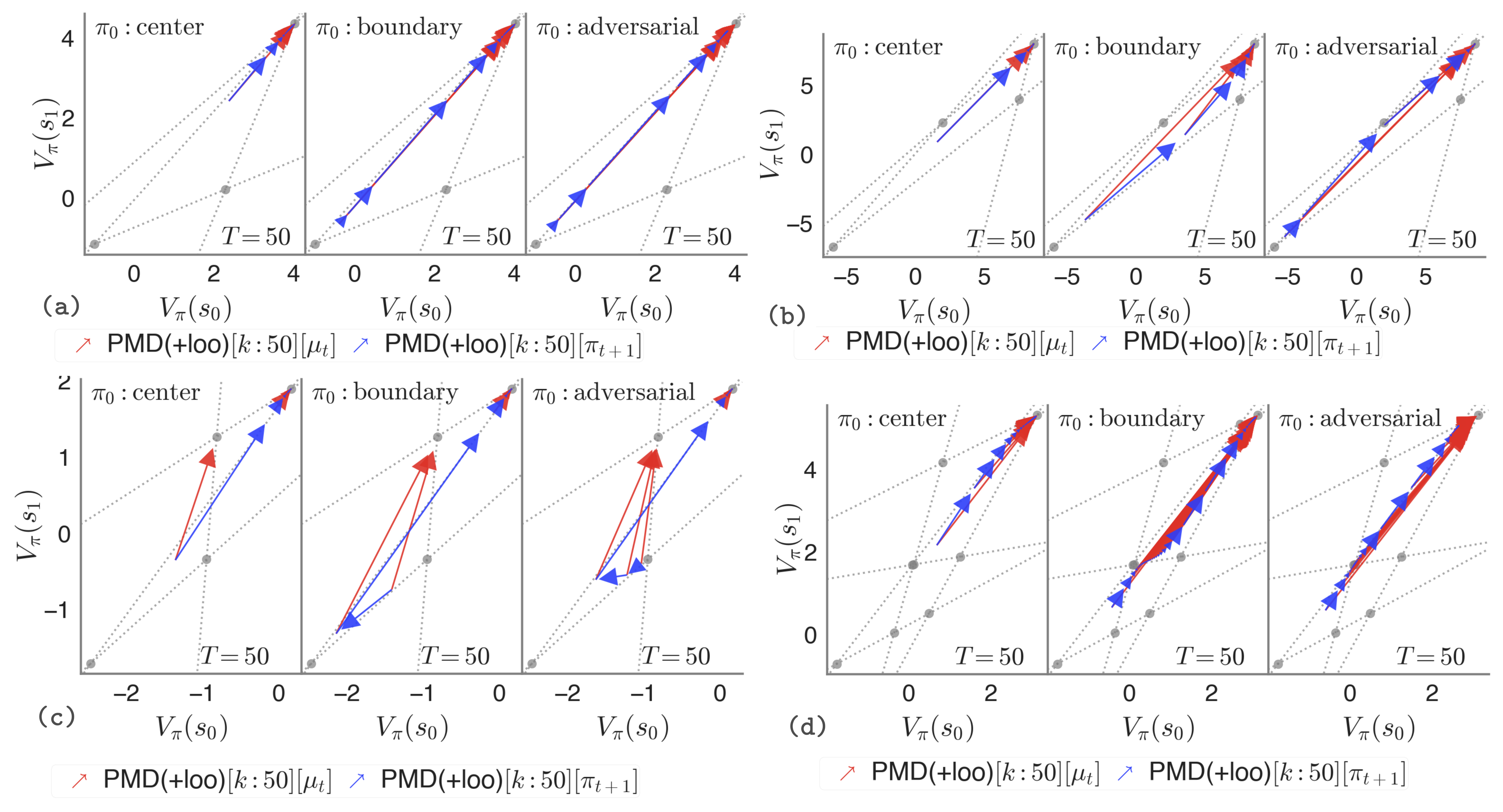}
    \caption{
    Shows the policy optimization dynamics of \pmdloo\ on the value polytope of the different example MDPs in Sec.~\ref{append:Details of two-state/action Markov Decision Processes}: \emph{(a)} example (ii), \emph{(b)} example (iii), \emph{(c)} example (i), \emph{(d)} example (iv). 
    }
    \label{fig:suppl_how__loo}
    \vspace{-10pt}
\end{figure}

\paragraph{\pmdloo}
{
In Fig.~\ref{fig:suppl_how__loo} we observe the dynamics of \pmdloo\ sometimes follow a different path through the polytope compared to \PMD\ or \PI\ (Fig.~\ref{fig:suppl_how__pi_pmd}), as they are following a different ascent direction, which may be more direct compared to that of \PI.  Compared to Fig.~\ref{fig:suppl_how__pmd_sweep_k}, in Fig.~\ref{fig:suppl_how__loo_sweep_n}, we see less accumulation points, and more jumps, i.e. the policy improvement step returns policies at further distance apart.
}
\begin{figure}[h]
        \centering
         \hspace{-10pt} 
        \includegraphics[width=1.\textwidth]{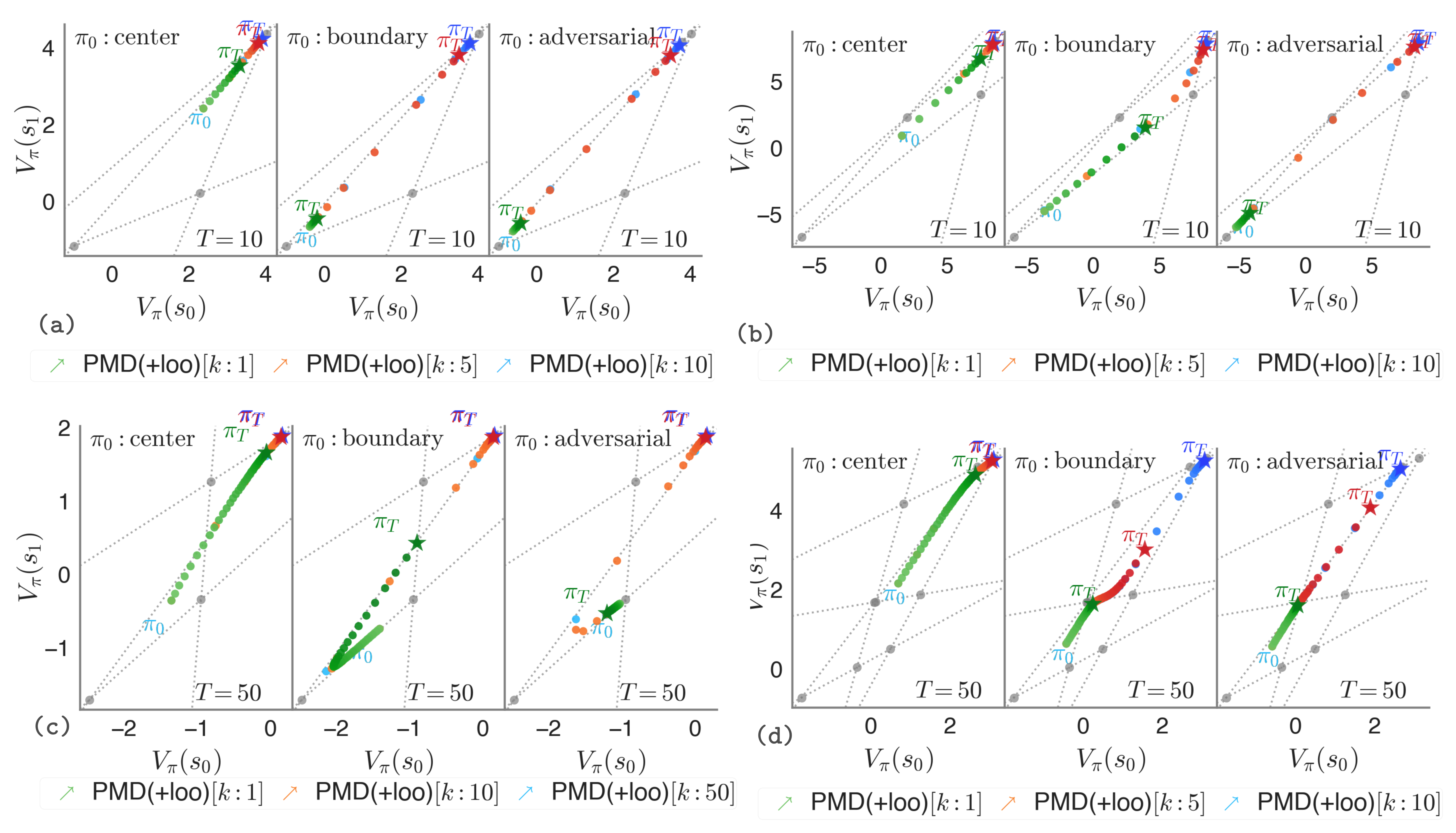}
    \caption{
    Shows the policy optimization dynamics of \pmdloo for different values of $k$ on the value polytope of the different example MDPs in Sec.~\ref{append:Details of two-state/action Markov Decision Processes}: \emph{(a)} example (ii), \emph{(b)} example (iii), \emph{(c)} example (i), \emph{(d)} example (iv). 
    }
    \label{fig:suppl_how__loo_sweep_n}
    \vspace{-10pt}
\end{figure}

\paragraph{\pmdmom}
{
Fig.~\ref{fig:suppl_how__pmd_mom} compares the policy dynamics of \PMD\ and \pmdmom\ and shows acceleration of the latter in those directions of ascent that align over consecutive steps, and have ill-conditioned optimization surfaces. 
}

\begin{figure}[h]
        \centering
         \hspace{-10pt} 
        \includegraphics[width=1.\textwidth]{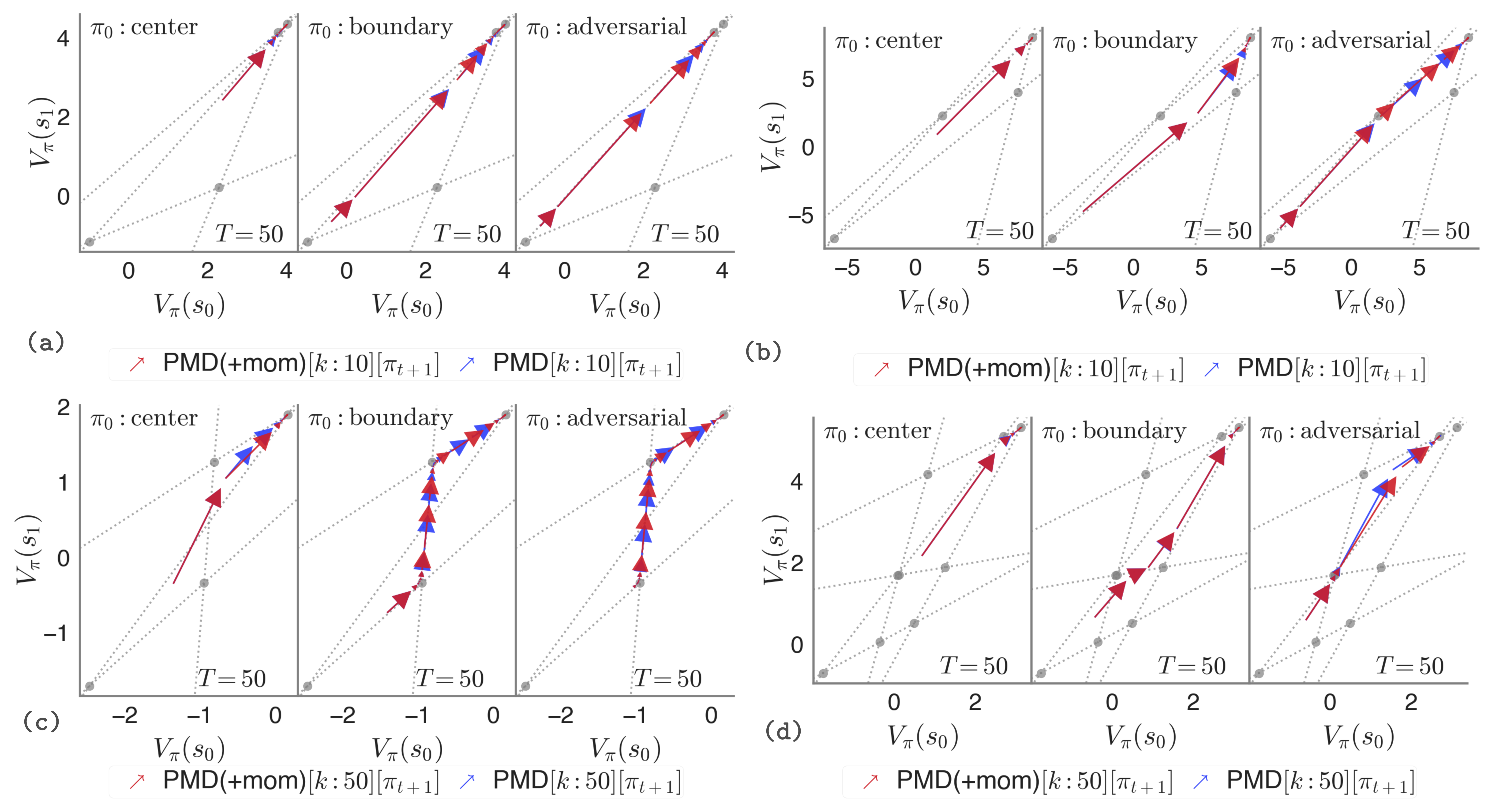}
    \caption{
    Compares the policy optimization dynamics of \PMD\ and \pmdmom\ on the value polytope of the different example MDPs in Sec.~\ref{append:Details of two-state/action Markov Decision Processes}: \emph{(a)} example (ii), \emph{(b)} example (iii), \emph{(c)} example (i), \emph{(d)} example (iv). 
    }
    \label{fig:suppl_how__pmd_mom}
    \vspace{-10pt}
\end{figure}


\subsection{Supplementary results for Sec.~\ref{sec:Acceleration with an inexact critic}}\label{append:Supplementary results for Acceleration with an inexact critic}

Fig.~\ref{fig:suppl_inexact_controlled_seeds_polytope_mdp3},~\ref{fig:suppl_inexact_controlled_seeds_polytope_mdp0},~\ref{fig:suppl_inexact_controlled_seeds_polytope_mdp1},~\ref{fig:suppl_inexact_controlled_seeds_polytope_mdp4} illustrate the variance over $50$ optimization trajectories initialized from a random uniform distribution with mean $0$ and standard deviation $1$, for each of the example two-state MDPs described in Appendix~\ref{append:Details of two-state/action Markov Decision Processes}. We use the same \emph{(controlled)} setting as in Sec.\ref{sec:Acceleration with an inexact critic}, and vary the critic's error $\tau$, and the policy approximation via $k$. 

The most illustrative example is Fig.~\ref{fig:suppl_inexact_controlled_seeds_polytope_mdp3}, since this setting presents the most ill-conditioned surface on which we can observe the impact of acceleration.
We observe in \emph{(a)} the instability of policy iteration with an inexact critic. As the critic's error grows the policy iterates start to oscillate between the corners of the polytope.

\PMD\ is better behaved for low $k$ values and starts to exhibit behaviour similar to \PI\ at larger $k$ values \emph{(c,d)}. We observe \pmdmom\ is more unstable than \PMD\ when presented with high level of errors in the inexact critic \emph{(f)}, and tends to stay more on the boundary of the polytope \emph{(e)}, which is consistent with having larger values of the gradient due to added momentum.

In \emph{(g,h)} the learning curves show the variance over trajectories stemming from the random initialization, the instability of \PI, the relative improvement of \pmdmom\ over \PMD, particularly striking for larger $k$ consistent with the theory. We observe in \emph{(h)} \pmdmom\ has more variance in the beginning, which may actually be desirable in terms of exploration, and that is achieves a similar performance to \PMD\ at the end of the optimization.

\begin{figure}[h]
        \centering
         \hspace{-10pt} 
        \includegraphics[width=1.\textwidth]{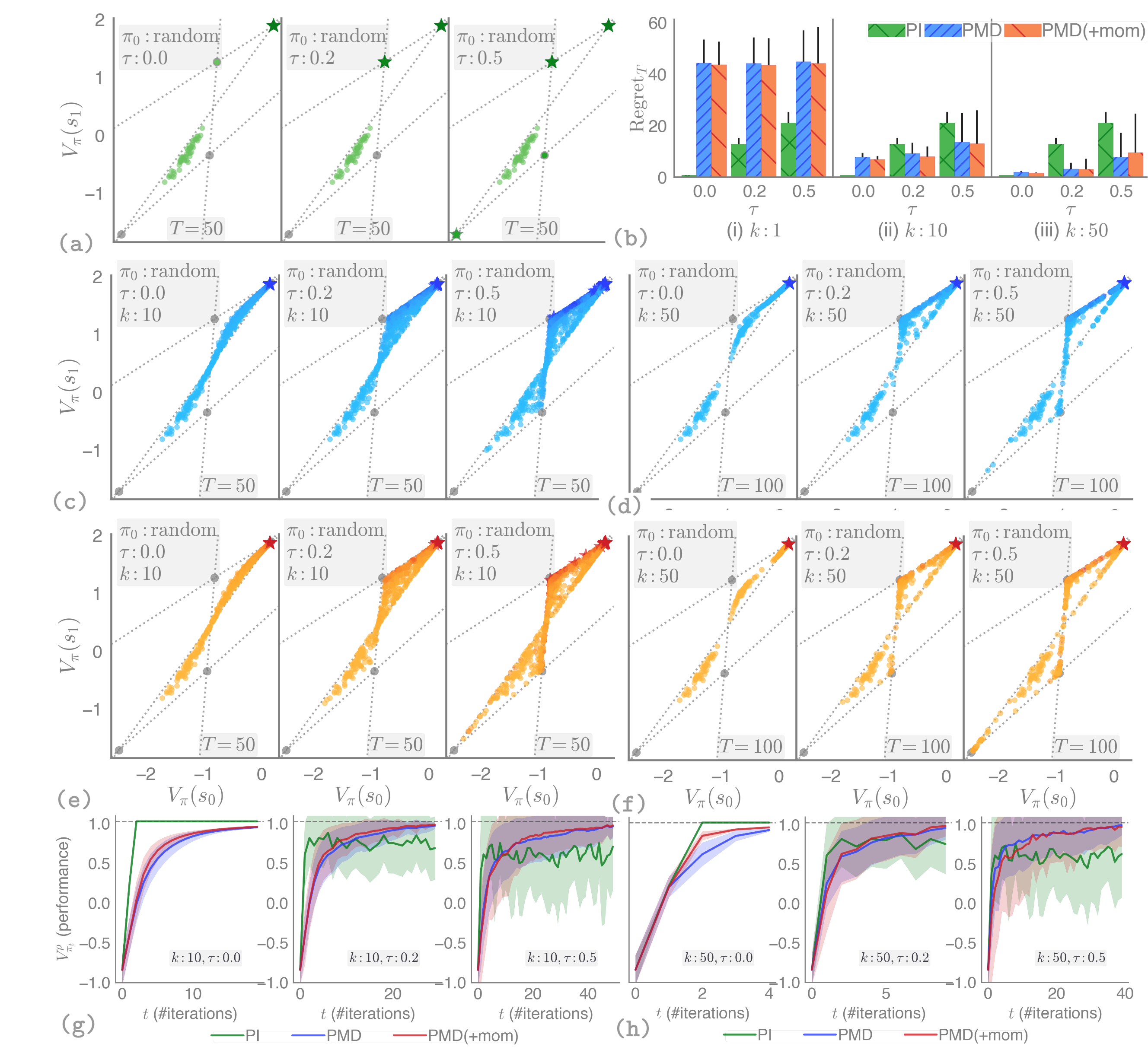}
    \caption{
    Compares the statistics of the policy optimization dynamics of \PI, \PMD\ and \pmdmom\ subject to variance from random initialization ($\pi_0: \emph{\text{random\_uniform}(0,1)}$, relative to the error in the inexact critic ($\tau$), and over different levels of policy approximation ($k$). Results correspond to example (i) from Sec.~\ref{append:Details of two-state/action Markov Decision Processes}. 
    }
    \label{fig:suppl_inexact_controlled_seeds_polytope_mdp3}
\end{figure}

\begin{figure}[h]
        \centering
         \hspace{-10pt} 
        \includegraphics[width=1.\textwidth]{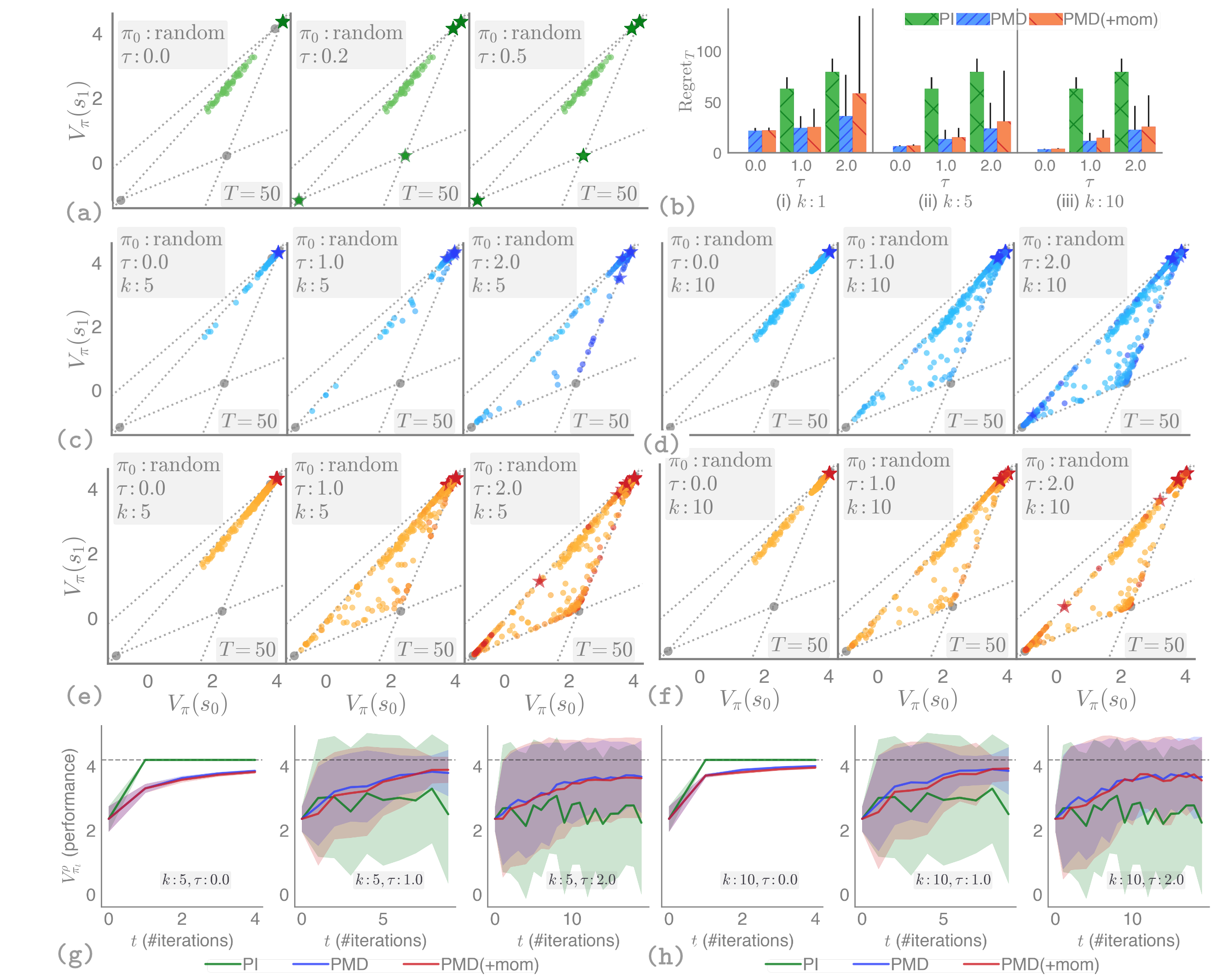}
    \caption{
    Compares the statistics of the policy optimization dynamics of \PI, \PMD\ and \pmdmom\ subject to variance from random initialization ($\pi_0: \emph{\text{random\_uniform}(0,1)}$, relative to the error in the inexact critic ($\tau$), and over different levels of policy approximation ($k$).  Results correspond to example (ii) from Sec.~\ref{append:Details of two-state/action Markov Decision Processes}.
    }
\label{fig:suppl_inexact_controlled_seeds_polytope_mdp0}
\end{figure}

\begin{figure}[h]
        \centering
         \hspace{-10pt} 
        \includegraphics[width=1.\textwidth]{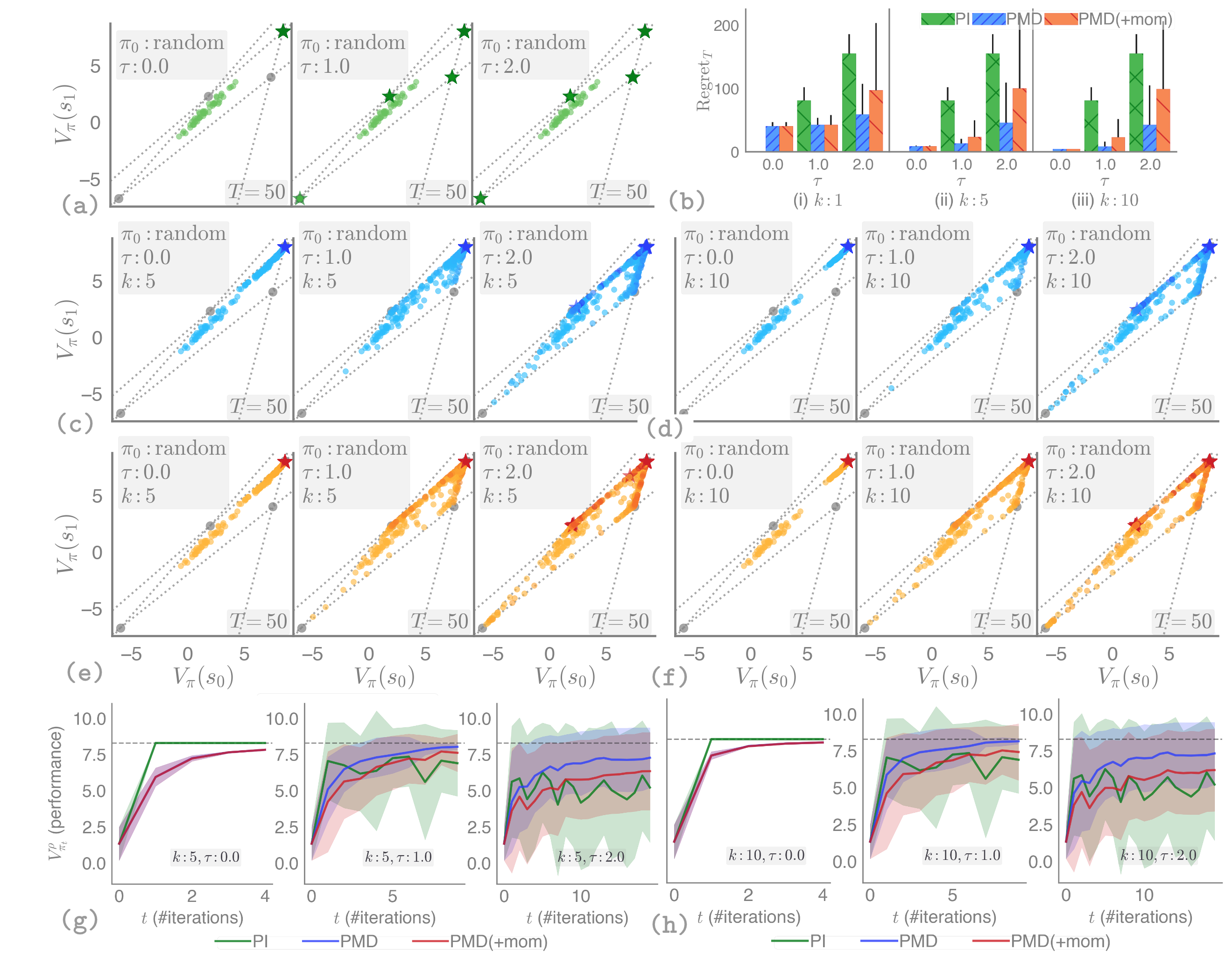}
    \caption{
    Compares the statistics of the policy optimization dynamics of \PI, \PMD\ and \pmdmom\ subject to variance from random initialization ($\pi_0: \emph{\text{random\_uniform}(0,1)}$, relative to the error in the inexact critic ($\tau$), and over different levels of policy approximation ($k$). Results correspond to example (iii) from Sec.~\ref{append:Details of two-state/action Markov Decision Processes}.
    }
\label{fig:suppl_inexact_controlled_seeds_polytope_mdp1}
\end{figure}

\begin{figure}[ht]
        \centering
         \hspace{-10pt} 
        \includegraphics[width=1.\textwidth]{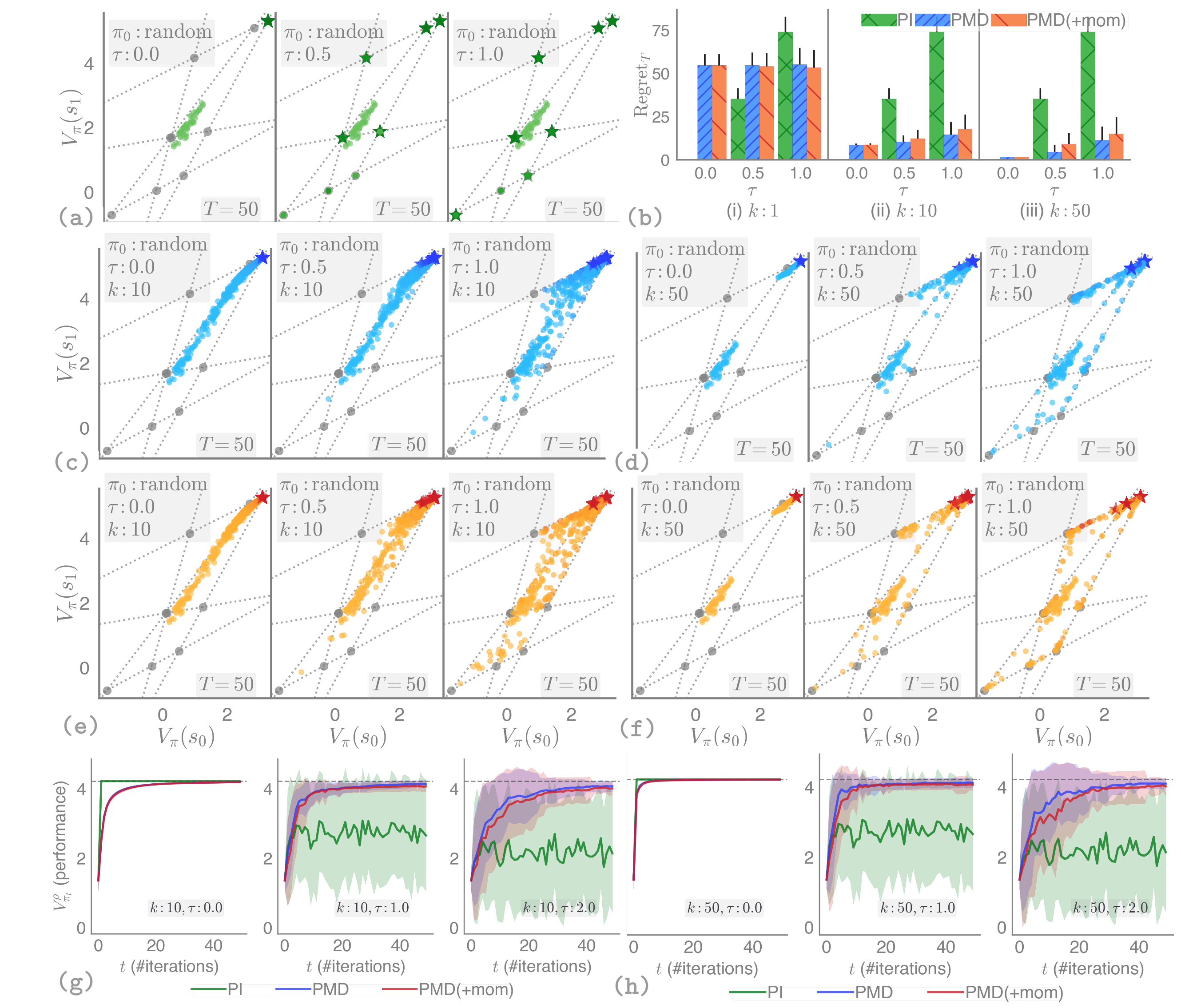}
    \caption{
    Compares the statistics of the policy optimization dynamics of \PI, \PMD\ and \pmdmom\ subject to variance from random initialization ($\pi_0: \emph{\text{random\_uniform}(0,1)}$, relative to the error in the inexact critic ($\tau$), and over different levels of policy approximation ($k$). Results correspond to example (iv) from Sec.~\ref{append:Details of two-state/action Markov Decision Processes}.
    }
\label{fig:suppl_inexact_controlled_seeds_polytope_mdp4}
\end{figure}
\clearpage


\end{document}